\documentclass{article} 
\usepackage{iclr2027_conference,times}

\usepackage{amsmath,amsfonts,bm}

\def\1{\bm{1}}

\DeclareMathAlphabet{\mathsfit}{\encodingdefault}{\sfdefault}{m}{sl}
\SetMathAlphabet{\mathsfit}{bold}{\encodingdefault}{\sfdefault}{bx}{n}

\def\gA{{\mathcal{A}}}
\def\gB{{\mathcal{B}}}

\def\gD{{\mathcal{D}}}

\def\gF{{\mathcal{F}}}
\def\gG{{\mathcal{G}}}
\def\gH{{\mathcal{H}}}

\def\gM{{\mathcal{M}}}
\def\gN{{\mathcal{N}}}
\def\gO{{\mathcal{O}}}

\def\gV{{\mathcal{V}}}
\def\gW{{\mathcal{W}}}

\def\sR{{\mathbb{R}}}

\newcommand{\E}{\mathbb{E}}

\newcommand{\R}{\mathbb{R}}

\newcommand{\KL}{D_{\mathrm{KL}}}

\DeclareMathOperator*{\argmin}{arg\,min}

\usepackage{hyperref}
\usepackage{url}
\usepackage{algorithm}
\usepackage{algpseudocode}
\usepackage{dsfont}
\usepackage{enumitem}
\setlist[itemize]{leftmargin=*, labelindent = 10pt}
\setlist[enumerate]{leftmargin=*, labelindent = 10pt}
\usepackage{autonum}
\title{Beyond Second-Order Information: \\ Paid Logistic Regression with Noisy Features}

\author{Sanghwa Kim\\
Kim Jaechul Graduate School of AI, KAIST\\
Seoul 02455, Republic of Korea\\
\texttt{tkdghk9667@kaist.ac.kr}
\setcounter{footnote}{1}
\And Kyoungseok Jang\thanks{Corresponding authors}\\
Department of AI, Chung-Ang University\\
Soeul 06974, Republic of Korea\\
\texttt{ksjang@cau.ac.kr}\\
\And
Se-Young Yun$^\dagger$\\
Kim Jaechul Graduate School of AI, KAIST\\
Seoul 02455, Republic of Korea\\
\texttt{yunseyoung@kaist.ac.kr}
}

\iclrpreprint

\usepackage{amsthm}

\newtheorem{theorem}{Theorem}
\newtheorem{lemma}[theorem]{Lemma}
\newtheorem{proposition}[theorem]{Proposition}
\newtheorem{corollary}[theorem]{Corollary}

\newtheorem{assumption}{Assumption}

\newtheorem{remark}[theorem]{Remark}

\newcommand{\Ber}{\mathrm{Ber}}

\newcommand{\Reg}{\mathrm{Reg}}
\newcommand{\TV}{\mathrm{TV}}
\newcommand{\ind}{\mathbf 1}

\begin{document}

\maketitle

\begin{abstract}
    We consider online binary prediction with paid stochastic features, where the learner can spend resources to obtain better observations of the underlying features. This setting is motivated by applications such as medical testing, where additional cost can provide more precise measurements of patient's features for predicting a binary outcome. At each round, a latent feature is drawn i.i.d., the response is generated through a logistic model, and the learner chooses how much to pay before observing a noisy version of the feature. The goal of learner is to jointly control prediction loss and acquisition cost. In contrast to the former result on the online linear regression with paid features, covariance information alone is no longer sufficient: two noise distributions with the same covariance may induce different optimal logistic losses, and lower variance can even lead to higher loss. Under a structural ordering in which larger payments yield better observations by removing independent noise, we obtain $\widetilde O(\sqrt{T})$ regret when the noise structure is known and $\widetilde O(T^{2/3})$ regret when it is unknown, with matching lower bounds up to logarithmic factors.
\end{abstract}

\section{Introduction}

Information rarely comes for free. In many real-world prediction problems, a decision maker must balance the quality of information used for prediction against the cost of acquiring it.
Consider a clinic that predicts each patient's probability of specific health outcome, which is often binary. Before observing a patient's features, the clinic must decide how much to spend on measuring them. 
More accurate tests or repeated measurements can improve prediction, but incur more cost. A similar trade-off arises when an online platform predicts whether a customer purchase a product or not: reliable customer information may improve the prediction, but acquiring it requires resources.

We study this problem as online logistic regression with paid noisy features. At each round, an individual arrives independently from a fixed but unknown population. Before observing the individual's feature, the learner chooses a payment that determines the quality of feature measurement, together with a prediction rule to apply to the resulting observation.
It then observes the noisy feature, predicts a probability, and observes a binary outcome.
Over the horizon, the the learner aims to learn both how much to pay and how to predict from noisy features, with the goal of minimizing cumulative logistic loss plus a payment penalty.

Several lines of work study related problems.
Classical errors-in-variables models study regression with corrupted covariates, while active or costly feature acquisition problem considers which features to purchase \citep{janisch2019classification}. In online learning, \citet{cesabianchi2011online} studied linear prediction with noisy features and known or unknown noise covariance, but without a cost-quality trade-off. \citet{vanderhoeven2023trading} studied binary classification with paid stochastic experts, where payments improve the reliability of binary expert predictions rather than the measurement of a continuous feature vector. The closest predecessor is \citet{merlis2026online}, who introduced online linear regression with paid stochastic features under squared loss. They established regret rates of $\widetilde O(\sqrt T)$ and $\widetilde O(T^{2/3})$ when the noise covariance profile is known and unknown, respectively, with matching lower bounds up to logarithmic factors.

Moving from linear to logistic prediction is natural for binary outcomes, but introduces new  statistical challenge. Related extensions have been studied in online logistic regression \citep{hazan2014logistic} and bandits \citep{filippi2010parametric,faury2020improved,abeille2021instance}. In our setting, the nonlinearity of the logistic function and loss changes how how noisy feature affects prediction loss. Under squared loss, the contribution of independent mean-zero feature noise is determined by its covariance. Under logistic loss, however, we show that two noise distributions with the same covariance can produce different prediction loss, and a smaller variance can even yield a larger loss. Thus, covariance alone does not determine the prediction quality of a paid measurement. 
Moreover, noisy feature can also change which predictor works best: a predictor that is optimal for clean features may not be optimal for noisy measurements. The learner therefore needs a prediction rule adapted to its chosen measurement quality.

Our goal is to learn a payment and prediction rule that jointly minimizes expected prediction loss and cost across the population.
We measure performance by expected cumulative regret relative to this optimal population loss. Regret captures the additional population loss incurred while learning the optimal trade-off from sequential observations. 
This leads us to the following questions:
\begin{center}
    \textit{How quickly can a learner approach the population-wise optimal trade-off\\ between the prediction quality and acquisition cost?}
\end{center}

\paragraph{Contributions.}
This paper makes the following contributions.
\begin{itemize}
\item \textbf{Second-order information is insufficient for paid logistic prediction (Proposition~\ref{prop:insufficient})}
We first identify a fundamental distinction from the former paid linear regression setting.
While the effect of additive feature noise under squared loss can be characterized through its covariance, this is no longer true under logistic loss.
We construct noise distributions with identical covariance but different optimized logistic losses.
More strikingly, a noise distribution with smaller variance can induce a strictly larger optimized prediction loss.
Thus, second-order information alone is in general insufficient to characterize the improvement of feature quality in logistic prediction.

\item \textbf{Order-optimal regret with a known degradation increment structure.}
We introduce a structural ordering of the noise in which the lower payment yields a noisier observation by adding independent degradation noise (Assumption~\ref{ass:degradation}). 
When the corresponding distribution of degradation increment is known, an observation obtained at a higher payment can be transformed into statistically valid observations at lower payment levels (Proposition~\ref{prop:simulated-concentration}). 
Exploiting this one-sided information sharing, we develop a one-sided simulation-based grid elimination algorithm achieving $\widetilde{\mathcal{O}}(\sqrt{T})$ regret (Theorem~\ref{thm:known-upper}). 
We prove a matching $\Omega(\sqrt{T})$ minimax lower bound (Theorem~\ref{thm:known-lower}), showing that the horizon dependence is optimal up to logarithmic factors. 

\item \textbf{Order-optimal regret with an unknown noise profile.}
When the payment-dependent noise profile is unknown, observations collected at one payment level cannot in general be reused at another. 
We develop an asynchronous grid lower confidence bound algorithm that directly learns the prediction--payment trade-off and achieves $\widetilde{\mathcal{O}}(T^{2/3})$ regret (Theorem~\ref{thm:unknown-upper}). 
We complement this result with an $\Omega(T^{2/3})$ lower bound (Theorem~\ref{thm:unknown-lower}), even for one-dimensional Gaussian noise with a decreasing variance profile. 

\end{itemize}

\paragraph{Notation.}
For a vector $x\in\R^d$, $\|x\|_2$ denotes its Euclidean norm. For symmetric matrices $A$ and $B$, we write $A\preceq B$ when $B-A$ is positive semidefinite. We write $U\overset{d}{=}V$ to denote equality in distribution. We use $\gO(\cdot)$ for the standard asymptotic order notation and $\widetilde{\gO}(\cdot)$ when logarithmic factors are ignored. For a positive integer $T$, let $[T] = \{1,2, \cdots, T\}$ and $[T]_0 = [T] \cup \{0\}$. For $S>0$, define the closed Euclidean ball in $\R^d$ by $\gB_d(S) = \{ x \in \sR^d : \| x \|_2 \leq S \}$.

\section{Problem Setting}\label{sec:problem-setting}

Let $\theta^\star \in \gB_d(S)$ be an unknown parameter, where the radius $S$ is known. At every round $t\in[T]$, the environment (privately) draws a feature $X_t \overset{i.i.d.}{\sim} D_X$, and generates a binary response $Y_t \mid X_t\sim \mathrm{Bernoulli}( \mu(X_t^\top \theta^\star ))$, where $\mu(z) = (1 + \exp(-z))^{-1}$. 
The clean feature $X_t$ is not revealed to the learner, and its distribution $D_X$ is fixed across rounds but unknown.
The learner selects a payment $c_t\in[0,1]$ and chooses a linear predictor $\nu_t \in \gV = \gB_d(S)$, then observes the corrupted feature $\widehat{X}_t = X_t + N_t(c_t)$, where $N_t (c_t) \sim D_N(c_t)$.
For any payment $c$, the distribution $D_N(c)$ is mean-zero and independent of $(X, Y)$ conditionally on the payment $c$ and previous history.

Let $\gH_{t-1} := \{ (c_s, \nu_s, \widehat{X}_s, Y_s, \Xi_s)\}_{s=1}^{t-1}$ denote the learner's history before round $t$, where $\Xi_s$ is the (possibly empty) collection of all additional quantities privately generated by the learner during round $s$. The learner determines $c_t$ using only $\gH_{t-1}$, and chooses the predictor $\nu_t $ before observing the noisy feature $\widehat{X}_t$. Precisely, $\nu_t$ is measurable with respect to $\sigma ( \gH_{t-1}, c_t )$. After observing $\widehat{X}_t$, the learner predicts the probability of the response to be $1$ by $\mu ( \widehat{X}_t^\top \nu_t )$. The response $Y_t$ is then revealed to the learner.

The predicted probability is evaluated through the logistic negative log-likelihood, which is the canonical loss for probabilistic binary prediction.
For $y \in \{0, 1\}$ and $z \in \sR$, the logistic loss $\ell$ is 
\begin{equation}
    \ell(y,z) := - y \log \mu(z) - (1 - y) \log \big( 1 - \mu(z) \big) = \log \big( 1 + \exp(z) \big) -  yz.
\end{equation}
For a fixed payment $c$, let $\widehat{X}(c) = X + N(c)$ be a corrupted random feature at payment $c$. The \textit{population loss} of the predictor $\nu$ and the cost $c$, measuring the expected logistic loss plus the payment penalty, is defined as
\begin{equation}
    \label{eq:population-loss}
    L(c,\nu)
    =
    \E \big[ \ell (Y,\widehat X (c)^\top\nu ) \big]
    +\lambda c.
\end{equation}
where $\lambda >0 $ is a cost sensitivity parameter that determines the relative importance of prediction accuracy and payment. Here, the expectation is taken over the population of $X$, measurement noise $N(c)$, and outcome $Y$.
For each payment $c$, we define the \textit{optimized population loss} at and \textit{optimal predictor} at payment $c$ as
\begin{equation}
    L^\star(c)=\min_{\nu\in \gV }L(c,\nu),
    \qquad
    \nu^\star(c)\in \argmin_{\nu\in\gV}L(c,\nu).
\end{equation}
Although the response follows the logistic model with parameter $\theta^\star$, the learner predicts from the noisy feature $\widehat{X}(c)$. Indeed, the conditional probability given this feature is $\mathbb{P} (Y=1 | \widehat{X} ) = \E [\mu(X^\top\theta^\star) | \widehat{X}]$, which is generally not representable as $\mu(\widehat{X}^\top \nu)$ for any $\nu$. Thus, $\theta^\star$ need not be the best logistic-linear predictor from the noisy feature \citep{kim2023contextual, stefanski1985covariate}.

We define \textit{(global) optimal population loss} by $L^\star = \inf_{c \in[0,1]} L^\star (c)$. This is the optimal expected loss per round over all fixed payments and predictors.
Since each new feature is independent of the past, and the payment and predictor are chosen before observing the feature, the conditional expected loss at round $t$ is $L(c_t,\nu_t) \geq L^\star$. We therefore compare the learner's expected cumulative population loss relative to $T L^\star$ and define the expected regret as 
\begin{equation}\label{eq:regret}
    \mathrm{Reg}_T := \E \left[ \sum_{t=1}^{T} \big\{ L(c_t, \nu_t) - L^\star \big\} \right] = \E \left[ \sum_{t=1}^T \big( \ell( Y_t ,\widehat X_t^\top\nu_t ) + \lambda c_t \big) \right] -T L^\star.
\end{equation}
The second equality follows from the tower property; see Appendix~\ref{app:regret} for proof.

Throughout, we adopt the following regularity assumption on the features and measurement noises.
\begin{assumption}[Sub-Gaussian feature and noise]\label{ass:sg}
    There exists a known constant $B > 0$ such that $\left\| \E[ X] \right\|_2 \leq B$.
    There is a known constant $R\ge1$ such that, for every payment $c\in[0,1]$ and every $u \in \R^d$, $\E \big[ e ^{ u^\top (  X_t - \E[X] ) } \mid \gH_{t-1} \big] \leq e^{\frac{R^2 \| u \|_2^2}{2}}$ and $\E \big[ e ^{ u^\top N_t (c) } \mid \gH_{t-1}, c \big] \leq e^{\frac{R^2 \| u \|_2^2}{2}}$.
\end{assumption}

\subsection{Beyond Covariance: Structure of Noise Profile}

In the linear regression with squared loss \citep{merlis2026online}, the effect of additive measurement noise is characterized entirely by its covariance.
For mean-zero noise $N$ independent of $(X,Y)$,
\begin{equation}
    \E \big[ ( ( X + N )^\top \nu - X^\top \theta^\star )^2 \big] = \E \big[ \big( X^\top(\theta^\star - \nu) \big)^2] + \nu^\top \mathrm{Cov}(N) \nu.
\end{equation}
Hence, for a fixed predictor $\nu$, the additional squared loss caused by measurement noise is completely captured by its covariance $\mathrm{Cov}(N)$.
On the other hand, \textbf{this reduction does not extend to the logistic loss.} To see this, we expand the logistic loss $\ell(y,z)$ as 
\begin{equation}
    \ell(y, z) = \log \big( 1 + \exp(z) \big) -yz = \log 2 + \left( \tfrac{1}{2} - y \right) z + \tfrac{1}{8} z^2 - \tfrac{1}{192} z^4 + O(z^6).
\end{equation}
Unlike the squared loss, logistic loss contains terms of order higher than $2$. Therefore, covariance alone may not be sufficient to characterize the logistic loss. The following proposition formalizes the insufficiency of the covariance: covariance ordering does not imply the logistic loss ordering.
\begin{proposition}[Insufficiency of covariance]\label{prop:insufficient}
    There exists one-dimensional problem instances with mean-zero scalar noises $N(c)$ such that $\mathrm{Var}(N(c_1)) = \mathrm{Var}(N(c_2))$, but the two payments have different optimized logistic losses. Moreover, there exist other instances such that $\mathrm{Var}(N(c_1)) > \mathrm{Var}(N(c_2))$ while the optimized logistic loss at $c_1$ is strictly smaller than at $c_2$. 
\end{proposition}
The proof is deferred to Appendix~\ref{app:insufficient}. Proposition~\ref{prop:insufficient} shows that covariance alone is insufficient to characterize the effect of measurement noise under logistic loss: even equal variances can lead to different losses, and variance ordering may not align with the loss ordering. Therefore, we impose the following distributional assumption on the noise across payment levels:
\begin{assumption}[Ordered additive degradation]\label{ass:degradation}
    For every $0\le c_1\le c_2\le1$, there exists a mean-zero random vector $\xi_{c_1,c_2}$, independent of $(X,Y,N(c_2))$, such that
    \begin{equation}\label{eq:degradation}
        N(c_1)\overset d=N(c_2)+\xi_{c_1,c_2}.
    \end{equation}
\end{assumption}
We call $\xi_{c_1,c_2}$ a \textit{degradation increment}, or simply \textit{increment}.
Assumption~\ref{ass:degradation} means that for $c_1 \leq c_2$ the noise $N(c_1)$ at $c_1$ can be obtained by adding an independent mean-zero \emph{increment} $\xi_{c_1,c_2}$ to the noise $N(c_2)$.

In general, Assumption~\ref{ass:degradation} is stronger than covariance monotonicity in the paid linear case \citep{merlis2026online}. However, the two conditions are equivalent for Gaussian measurement noise.
In particular, suppose $N(c) \sim \gN (0, \Sigma(c))$ and $c_1 \leq c_2$ implies $\Sigma(c_2) \preceq \Sigma(c_1)$, then Assumption~\ref{ass:degradation} holds by taking $\xi_{c_1, c_2} \sim \gN (0, \Sigma(c_1) - \Sigma(c_2))$ independently of $N(c_2)$.

Assumption~\ref{ass:degradation} restores the natural ordering of loss. Since $\ell (y, z)$ is convex in $z$, for every $c_1\leq c_2$ and $\nu\in\gV$, Jensen's inequality conditional on $X, Y, N(c_2)$ and the fact $\E[\xi_{c_1,c_2}]=0$ yields
\begin{equation}\label{eq:logistic-loss-ordering}
    \E \Big[ \ell \big(Y, \big(X + N(c_1) \big)^\top \nu \big) \Big] = \E \Big[ \ell \big( Y, \big( X + N(c_2) + \xi_{c_1, c_2} \big)^\top \nu \big) \Big] \geq \E \Big[ \ell \big(Y, \big( X + N(c_2) \big)^\top \nu \big) \Big].
\end{equation}
That is, the prediction loss is nonincreasing in payment $c$.
Therefore, any increase in the population loss from $c_1$ to $c_2$ can only come from the additional payment, by at most $\lambda(c_2 - c_1)$. This directly gives the following one-sided regularity property of the population loss (see Appendix~\ref{app:one-side} for proof).
\begin{proposition}[One-sided Lipschitzness]\label{prop:one-side}
    Under Assumption~\ref{ass:degradation}, for every $0\leq c_1\leq c_2\leq1$ and every $\nu\in\gV$, $L(c_2, \nu) \leq L(c_1, \nu) + \lambda (c_2 - c_1)$. Consequently, $L^\star(c_2) \leq L^\star(c_1) + \lambda ( c_2 - c_1)$.
\end{proposition}

We consider two information regimes according to whether the distributions of degradation increments $\xi_{c_1,c_2}$ in Assumption~\ref{ass:degradation} are available to the learner. In the \emph{known-profile} setting, the learner knows the distribution of every degradation increment $\xi_{c_1,c_2}$ ($c_1<c_2$) and has access to an exact sampler from this distribution.
In the \emph{unknown-profile} setting, Assumption~\ref{ass:degradation} is known to hold, but the increment distributions of increments $\xi_{c_1,c_2}$ are unavailable to the learner.

\section{Online Logistic Regression with Known Noise Profile}\label{sec:known}

In this section, we consider the \emph{known-profile} setting. For every pair of payments $c_1 \leq c_2$, the learner knows the distribution of the increment $\xi_{c_1,c_2}$ in Assumption~\ref{ass:degradation} and has an exact sampler from this distribution. Thus, after paying $c_2$ and observing $\widehat{X}(c_2)$, the learner can independently draw $\xi_{c_1,c_2}$, then $\widehat{X}(c_2) + \xi_{c_1, c_2}$ has the same distribution as a noisy feature $\widehat{X}(c_1)$ at payment $c_1$.

To exploit this property, we first establish a concentration bound for the empirical logistic loss  when all observations are collected at a fixed payment level $c$. We then show that observations at adaptively chosen higher payments can be used to generate simulated samples with exactly the same distribution as the observations at a fixed lower payment.
These results form the statistical basis of the elimination algorithm introduced afterward.

\subsection{Uniform concentration at a fixed payment}

We begin with the simplest case in which the learner pays a fixed payment $c \in [0,1]$. In this case, the observed pairs $(\widehat{X}_s, Y_s)_{s=1}^{t}$ are independent and identically distributed samples of $(\widehat{X}(c), Y)$. We therefore establish a uniform concentration of empirical loss to its population counterpart $L(c, \nu)$ at round $t$ over all $\nu\in\gV$, so that the bound remains valid for any data-dependent predictor.


\begin{proposition}[Uniform concentration at a fixed payment]\label{prop:concentration}
    Fix $c\in[0,1]$, and let $(\widehat X_s(c),Y_s)_{s=1}^{t}$ be i.i.d. samples of $(\widehat X(c),Y)$. Under Assumption~\ref{ass:sg}, there is a universal constant $C_0>0$ such that, for every fixed $t \geq 1$ and every $\delta\in(0,1)$,
    \begin{equation}\label{eq:fixed-prob}
        \mathbb{P} \left( \sup_{\nu \in \gV} \left| \frac{1}{t} \sum_{s=1}^{t} \ell \! \left( Y_s, \widehat{X}_s (c)^\top \nu \right) +\lambda c - L(c, \nu) \right| \leq \gamma ( t ,\delta) \right) \geq 1 - \delta, 
    \end{equation}
    where $\gamma(t, \delta) = C_0 S (B+R) \sqrt{(d + \log(1/\delta))/t}$.
\end{proposition}
\begin{proof}[Proof sketch]
    The centered logistic loss is a contraction in prediction score $\widehat{X}(c)^\top \nu$. By Rademacher symmetrization and contraction inequality, the uniform deviation of empirical loss can be controlled by Rademacher-symmetrized feature sum $ \sum_{s=1}^{t} \varepsilon_s \widehat{X}_s(c)$, where $(\varepsilon_s)_{s=1}^{t}$ are Rademacher random variables. Under Assumption~\ref{ass:sg}, this sum is a sub-Gaussian random vector. Therefore, a moment bound for the Euclidean norm of sub-Gaussian vector yields the desired uniform concentration bound. The full proof is given in Appendix~\ref{app:concentration}.
\end{proof}

Proposition~\ref{prop:concentration} applies when all observations at round $t$ are collected at the same payment $c$. 
However, in the online learning setting, observations may be collected at adaptively chosen payments.
We next show that high-payment observations can be transformed into a \emph{simulated observation} at a lower target payment, and derive a corresponding concentration bound for these simulated observations.

\subsection{Validity of Simulated Lower-payment Observations}
\label{sec:simulated}

At round $t$, the learner selects $c_t$ and observes $\widehat{X}_t = X_t + N_t(c_t)$. For any payment $c \leq c_t$, Assumption~\ref{ass:degradation} gives $N(c) \overset{d}{=} N(c_t) + \xi_{c,c_t}$.
Since the distribution of $\xi_{c, c_t}^{(t)}$ is known, the learner can sample $\xi_{c, c_t}$  independently of $(\widehat{X}_t, Y_t)$ and construct a simulated feature $\widetilde{X}_t (c) := \widehat{X}_t + \xi_{c, c_t}^{(t)}$.
Conditional on the history $\gH_{t-1}$ and the current payment $c_t$, the pair of simulated observation satisfies
$(\widetilde{X}_t (c), Y_t) \overset{d}{=} (\widehat{X}(c), Y)$. Thus, an observation collected at the payment $c_t$ provides a valid simulated observation for every lower payment $c \leq c_t$.

We next consider a sequence of such simulated observations for a fixed target payment $c$. Define $\tau(c) = \inf \{ t \geq 1 : c_t < c \}$ with $\tau(c)=\infty$ if the learner never selects a payment below $c$. If $\tau(c) > t$, $c_1,\ldots, c_t \geq c$, so all of the first $t$ observations can be converted into a simulated observation at target payment $c$.
Although the payments $c_t$ may vary and be selected adaptively, the conditional distributional equivalence allows us to inherit the result of Proposition~\ref{prop:concentration} as follows:
\begin{proposition}[Concentration of simulated observations]\label{prop:simulated-concentration}
    Fix $c \in [0,1]$ and let $\tau(c) = \inf \{ t \geq 1 : c_t < c\}$. For each $s < \tau(c)$, independently sample $\xi^{(s)}_{c, c_s}$ and construct a simulated feature as $\widetilde{X}_s (c)= \widehat{X}_s + \xi^{(s)}_{c, c_s}$. Then, under Assumption~\ref{ass:sg} and \ref{ass:degradation}, for every fixed $t \geq 1$ and $\delta \in (0,1)$, 
    \begin{equation}
        \mathbb{P} \left( t < \tau(c), \quad \sup_{\nu \in \gV} \left| \frac{1}{t}  \sum_{s=1}^{t}\ell \! \left(Y_s,\widetilde{X}_s(c)^\top\nu\right)
        +\lambda c-L(c,\nu)\right|  > \gamma(t, \delta) \right) \leq \delta.
    \end{equation}
\end{proposition}
\begin{proof}[Proof sketch]
    The proof completes the sequence after $\tau(c)$ with independent virtual observations drawn from the distribution of $(\widehat{X}(c), Y)$. Hence, each simulated observation pair before $\tau(c)$ also has this same conditional distribution regardless of preceding history. Therefore, the completed sequence is i.i.d. with the same distribution as $(\widehat X(c),Y)$. Since the event that $t<\tau(c)$ and the concentration bound fails for the simulated observations is contained in the failure event for the i.i.d. sequence above, Proposition~\ref{prop:concentration} bounds this probability by $\delta$. The full proof is given in Appendix~\ref{app:simulated-concentration}.
\end{proof}
Proposition~\ref{prop:simulated-concentration} shows that, as long as the first $t$ payments are all at least $c$, the simulated observations $(\widetilde{X}_s, Y_s)_{s=1}^{t}$ behave as if they have been collected directly at $c$, even when the actual payments $c_t$ are chosen adaptively.
This result is a key idea behind the known-profile algorithm: an observation collected at a sufficiently high payment can be reused to calculate the empirical loss at lower payment levels, with the same concentration rate as direct observations.

\subsection{One-Sided Simulation-Based Grid Elimination Algorithm}\label{sec:alg}

The algorithm discretizes the continuous payment interval $[0,1]$ into the uniform grid $c^{(k)}=k/K$, $k \in [K]_0$, and maintains an active set $\gA \subseteq [K]_0$ of payment indices that have not yet been ruled out as optimal. 
The algorithm works in sequential batches.
At the beginning of each batch, it selects the largest active index $j = \max \gA$ from the active set and keeps paying $c^{(j)}$ throughout the batch.
Since every active index $k \in \gA$ satisfies $c^{(k)} \leq c^{(j)}$, each observation at $c^{(j)}$ can be converted into one valid simulated observation for every active payment $c^{(k)}$. Thus, all active payments accumulate observations at the same rate and shares a common sample count. 
Each batch ends when this sample count reaches the next power of $2$ or $T$. Unless the batch terminates at $T$, empirical losses and predictors for each $k\in\gA$ are updated, and we call the corresponding sample count a \textbf{checkpoint}.
The pseudocode of algorithm is given in Algorithm~\ref{alg:known}.

\begin{algorithm}[t]
\caption{One-Sided Simulation-Based Grid Elimination}
\label{alg:known}
\begin{algorithmic}[1]
\Require Horizon $T$, grid parameter $K$, confidence level $\delta$, known increment distributions.
\State $\gA \gets \{0,1,\ldots,K\}$, $n\gets0$.
\State Set $\widehat\nu_{k,0}=0$ for all $k\in\gA$.
\While{$n<T$}
    \State $b\gets\min\{\max\{1,n\},T-n\}$ and $j\gets\max\gA$.
    \For{$s=1,\ldots,b$}
        \State Pay $ c^{(j)}$, and set $\nu\gets \widehat{\nu}_{j,n}$.
        \State Observe $\widehat{X}$, predict $\mu(\widehat{X}^\top \nu)$, and observe $Y$.
        \For{each $k\in\gA$}
            \If{$k=j$}
                \State $\widetilde X_k\gets\widehat X$
            \Else
                \State Draw $\xi_{c^{(k)}\!,c^{(j)}}$ and set $\widetilde X_k\gets\widehat X+\xi_{c^{(k)}\!,c^{(j)}}$.
            \EndIf
            \State Append $(\widetilde X_k,Y)$ to the data stream of payment $c^{(k)}$.
        \EndFor
    \EndFor
    \State $n\gets n+b$.
    \If{$n<T$}
        \For{each $k\in\gA$}
            \State Compute $\widehat\nu_{k,n}$ and $\widehat L_{k,n}^\star$.
        \EndFor
        \State $U\gets\min_{k\in\gA}\{\widehat L_{k,n}^\star+\gamma(n,\delta_T)\}$.
        \State $\gA\gets\{k\in\gA:\widehat L_{k,n}^\star-\gamma(n,\delta_T)\le U\}$.
    \EndIf
\EndWhile
\end{algorithmic}
\end{algorithm}

Suppose a batch after checkpoint $n$ begins with active set $\gA$.
At each round in the batch, the learner pays $c^{(j)}$ with $j = \max \gA$ and uses the predictor $\hat{\nu}_{j, n}$ computed at the checkpoint.
After observing $(\widehat{X}_s, Y_s)$ at round $s$, for every active $k < j$, the learner independently samples $\xi^{(s)}_{c^{(k)}\!,c^{(j)}}$ and constructs $\widetilde X_{k,s} = \widehat{X}_s + \xi^{(s)}_{c^{(k)}\!,c^{(j)}}$; while $\widetilde X_{j,s} = \widehat{X}_s$. The pair $(\widetilde X_{k,s} ,Y_s)$ is then appended to the data stream associated with payment $c^{(k)}$.

At the end of the batch after checkpoint $n\geq1$ that does not terminate at $T$, the new common sample count is $m = 2n$. For each active payment $c^{(k)}$, define the empirical losses and predictors as
\begin{equation}
    \widehat{L}_{k,m}(\nu) := \frac{1}{m} \sum_{s=1}^{m} \ell \big( Y_s, \widetilde{X}_{k,s}^\top \nu \big) + \lambda c^{(k)}, \quad \widehat{\nu}_{k,m} \in \argmin_{\nu \in \gV} \widehat{L}_{k,m}(\nu);
\end{equation}
and denote optimized empirical loss by $\widehat{L}^\star_{k,m} : = \widehat{L}_{k,m}(\widehat{\nu}_{k,m})$.
The algorithm compares the active payments using confidence intervals of optimized empirical losses. We set $\delta_T = \delta/ \big( (K+1)\lceil \log_2 T\rceil \big)$ and define $U := \min_{k \in \gA} \{ \widehat{L}^\star_{k,m} + \gamma(m, \delta_T) \}$. A payment $c^{(k)}$ remains active if $\widehat L^\star_{k,m} - \gamma(m,\delta_T) \leq U$ and is eliminated otherwise.
That is, $c^{(k)}$ is eliminated if its optimistic loss exceeds the smallest pessimistic loss among the active payments.

We next show that shared simulated observations, together with elimination, achieves $\widetilde O(\sqrt T)$ regret.
\begin{theorem}[Regret bound of Algorithm~\ref{alg:known}]
    \label{thm:known-upper}
    Suppose Assumptions~\ref{ass:sg} and \ref{ass:degradation} hold and the degradation increment distributions are known to the learner. For every $T\ge2$, $K\ge1$, and $\delta\in(0,1)$, Algorithm~\ref{alg:known} satisfies
    \begin{equation}\label{eq:known-upper}
        \mathrm{Reg}_T  \leq C_1 S (B+R) \sqrt{T D_T} + \frac{\lambda T }{K} + \log 2 + \lambda +\delta T \big[ \log 2 + S (B + 2R) + \lambda \big],
    \end{equation}
    where $C_1>0$ is a universal constant and $D_T = d + \log (1/\delta_T)$.
\end{theorem}
\begin{proof}[Proof sketch]
    We define a good event $\gG$ on which uniform concentration holds simultaneously for all grid payment and checkpoints. By Proposition~\ref{prop:simulated-concentration} and a union bound, $\mathbb P(\gG)\geq1-\delta$. On $\gG$, the best grid payment $c^{(k^\star)}$ is never eliminated, and every surviving payment $c^{(k)}$ satisfies $L(c^{(k)}, \widehat{\nu}_{k,n})-L^\star(c^{(k^\star)}) = O(\gamma(n,\delta_T))$. Therefore, during the batch following checkpoint $n$, the instantaneous regret relative to $L^\star(c^{(k^\star)})$ is $O(\gamma(n,\delta_T) )$.
    Since this batch has at most $n$ rounds, regret during this batch is $O(n\gamma(n,\delta_T)) = O ( S(B+R)\sqrt{nD_T} )$.
    Since the sample count doubles after batches, summing over batches yields regret of $O ( S(B+R)\sqrt{TD_T} )$. The first round contributes at most $\log 2 + \lambda$.
    Finally, Proposition~\ref{prop:one-side} implies $L^\star(c^{(k^\star)}) - L^\star \leq \lambda/K$, so discretization contributes at most $\lambda T /K$ to the regret. On $\gG^c$, the regret is controlled by the uniform per-round loss bound. Combining these terms gives the result.
    The full proof is given in Appendix~\ref{app:known-upper}
\end{proof}
The leading terms in the regret bound are the statistical error $S(B+R) \sqrt{T D_T}$ and the discretization error $\lambda T / K$. Balancing these two terms suggests choosing $K \asymp \frac{\lambda \sqrt{T}}{S(B+R) \sqrt{D_T}}$. Taking $\delta=1/T$ makes the contribution of failure event $O(1)$. Since $D_T$ depends only logarithmically on $T$, $K$ and $1/\delta$, $D_T = \widetilde{O}(d)$. Therefore, by treating $S$, $B$, $R$, and $\lambda>0$ as fixed constants, the expected regret satisfies $\mathrm{Reg}_T = \widetilde{O} ( \sqrt{d T} )$. 

\begin{remark}
    The simulated observations across different payments are generally dependent since they are generated from the same actual observation. This dependence does not affect the analysis since Proposition~\ref{prop:simulated-concentration} is applied separately to each payment $c^{(k)}$, and their concentration events are combined only through a union bound, which does not require independence across payments.
\end{remark}
\begin{remark}
    A key advantage of the known-profile setting is that the statistical error does not scale polynomially with the number of grid points $K$. Since every observation collected at the largest active payment can be degraded to every active lower payment, all active payments are learned simultaneously. Consequently, increasing $K$ primarily increases computation cost and only logarithmically affects the statistical error, while reducing the approximation error caused by discretization.
\end{remark}

We now show that our algorithm is minimax optimal in its dependence on $T$, up to logarithmic factor, by establishing a matching $\Omega(\sqrt{T})$ minimax lower bound for one-dimensional instances (see Appendix~\ref{app:known-lower} for full proof and construction).
\begin{theorem}[Minimax lower bound in known-profile setting]\label{thm:known-lower}
    There exist constants $S$, $B$, $R$, $\lambda$ and $C_k > 0$ such that, for all sufficiently large $T$, every possibly randomized online algorithm incurs
    \begin{equation}
        \mathrm{Reg}_T \geq C_k \sqrt{T}
    \end{equation}
    on some one-dimensional instance with a known noise profile, where Assumptions~\ref{ass:sg} and \ref{ass:degradation} hold.
\end{theorem}

\begin{remark}
    In classical linear and logistic regression tasks, logarithmic regret in $T$ is achievable in both settings \citep{vovk1997competitive,hazan2014logistic}. In the paid regression setting, \citet{merlis2026online} showed that introducing a payment--information trade-off fundamentally changes the picture even for linear regression, yielding an $\Omega(\sqrt{T})$ lower bound. Theorem~\ref{thm:known-lower} estalishes the same $\Omega(\sqrt{T})$ lower bound for paid logistic regression.
\end{remark}

\section{Extension to Unknown Noise Profile}\label{sec:unknown}

In the \emph{unknown-profile} setting, Assumption~\ref{ass:degradation} is known to hold, but neither the distribution of increment $\xi_{c_1,c_2}$ nor its sampler is available to the learner.
Therefore, an observation obtained at one payment cannot be converted into a valid observation at another payment. Unlike the known-profile setting in Section~\ref{sec:known}, each grid payment must be learned from its own observations.

\subsection{Asynchronous Grid Lower Confidence Bound Algorithm}
The algorithm uses the same grid as in Section~\ref{sec:known}.
For each grid payment $c^{(k)}=k/K$, $k\in [K]_0$, let $n_t(k) := \sum_{s=1}^{t} \mathds{1} [k_s = k]$ be the number of rounds the payment $c^{(k)}$ has been selected up to round $t$. The algorithm maintains a separate data stream for each grid and stores all actual observations at that payment. Consequently, the sample counts $n_t(k)$ evolve asynchronously across payments.

The algorithm first selects every grid payments once and uses $\nu_t=0$. At round $t$, the empirical loss and minimizer for payment $c^{(k)}$ is recomputed whenever its sample count $n_t(k)$ reaches a power of $2$. We call this sample count a \textbf{local checkpoint} of payment $c^{(k)}$ and denote its latest local checkpoint before round $t$ by $m_t(k) := 2^{\lfloor\log_2 n_{t-1}(k) \rfloor}$, then $m_t(k) \leq n_{t-1}(k) < 2 m_t(k)$.

Let $(\widehat{X}_{k,s}, Y_{k,s})_{s=1}^{m}$ denote the observations collected at $c^{(k)}$ up to its local checkpoint $m$ , then the empirical losses and predictors are defined as
\begin{equation}
    \widehat{L}_{k,m} (\nu) := \frac{1}{m} \sum_{s=1}^{m} \ell \big( Y_{k, s}, \widehat{X}_{k, s}^\top \nu \big) + \lambda c^{(k)}, \quad \widehat{\nu}_{k, m} \in \argmin_{\nu \in\gV} \widehat{L}_{k,m} (\nu);
\end{equation}
and denote optimized empirical losses by $\widehat{L}^\star_{k, m} : = \widehat{L}_{k,m} (\widehat{\nu}_{k,m})$.
At each round $t \geq K+2$, the algorithm optimistically selects $k_t \in \argmin_{k\in[K]_0} \{\widehat{L}_{k, m_t(k)}^\star - \gamma(m_t(k), \delta_T)\}$,
then pays $c^{(k_t)}$ and uses the predictor $\nu_t = \widehat{\nu}_{k_t, m_t(k_t)}$.
After observing $(\widehat{X}_t, Y_t)$, the pair is appended to the data stream of payment $c^{(k_t)}$.
The pseudocode is given in Algorithm~\ref{alg:unknown}.

We show that this asynchronous lower-confidence bound algorithm achieves $\widetilde{O}(T^{2/3})$ regret.

\begin{algorithm}[t]
\caption{Asynchronous Grid Lower Confidence Bound}
\label{alg:unknown}
\begin{algorithmic}[1]
\Require Horizon $T$, grid parameter $K$, confidence level $\delta$.
\For{$k=0,1,\ldots,K$}
    \State Pay $c^{(k)}$, predict $1/2$ ($\nu=0$), observe $(\widehat X,Y)$, and store the sample for $c^{(k)}$.
    \State Set $n(k)\gets1$, $m(k)\gets1$.
    \State Compute $\widehat\nu_{k,1}$ and $\widehat L_{k,1}^\star$.
\EndFor
\For{$t=K+2,\ldots,T$}
    \State Choose $k_t\in\argmin_{k\in[K]_0} \{\widehat L_{k,m(k)}^\star-\gamma(m(k),\delta_T)\}$.
    \State Pay $c^{(k_t)}$, and set $\nu_t \gets \widehat\nu_{k_t,m(k_t)}$.
    \State Observe $\widehat{X}_t$, predict $\mu(\widehat X_t^\top\nu_t)$, and observe $Y_t$
    \State Append $(\widehat{X}_t, Y_t)$ to the data stream of $c^{(k_t)}$ and set $n(k_t)\gets n(k_t)+1$.
    \If{$n(k_t)=2m(k_t)$}
        \State $m(k_t)\gets2m(k_t)$.
        \State Recompute $\widehat\nu_{k_t,m(k_t)}$ and $\widehat L_{k_t,m(k_t)}^\star$.
    \EndIf
\EndFor
\end{algorithmic}
\end{algorithm}

\begin{theorem}[Regret bound of Algorithm~\ref{alg:unknown}]\label{thm:unknown-upper}
    Suppose Assumptions~\ref{ass:sg} and \ref{ass:degradation} hold. For $T\ge2$, integer $1\le K<T$, and $\delta\in(0,1)$, the Algorithm~\ref{alg:unknown} satisfies
    \begin{equation}\label{eq:unknown-upper}
        \begin{aligned}
        \mathrm{Reg}_T & \leq C_2 S (B+R) \sqrt{(K+1)T D_T} + \frac{\lambda T}{K} +(K+1) (\log 2 +  \lambda)\\
        & \qquad + \delta T \big[ \log 2 + S(B+2R) + \lambda \big],
        \end{aligned}
    \end{equation}
    where $C_2>0$ is a universal constant and $D_T = d + \log (K+1) + \log \lceil \log_2 T \rceil + \log(1/\delta)$.
\end{theorem}
\begin{proof}[Proof sketch]
    We define a good event $\gG$ on which uniform concentration holds simultaneously for all grid payments and all possible local checkpoints.
    On $\gG$, the selection rule implies that the payment $c^{(k_t)}$ and predictor $\widehat{\nu}_{k_t, m_t(k_t)}$ satisfy $L(c^{(k_t)}, \widehat{\nu}_{k_t,m_t(k_t)})-L^\star(c^{(k^\star)}) = O(\gamma(m_t(k_t),\delta_T))$.
    We sum this regret separately for each payment over all possible batches. For each payment $c^{(k)}$, the cumulative regret is $O (S (B+R) \sqrt{n_{T}(k) D_T})$. Summing over all $K+1$ grid payments and Cauchy--Schwarz inequality yield the regret of $O( S(B+R)\sqrt{(K+1) T D_T} )$. The remaining steps are the same as in the proof of Theorem~\ref{thm:known-upper}. The full proof is given in Appendix~\ref{app:unknown-upper}.
\end{proof}

The statistical error $S(B+R)\sqrt{K T D_T}$ and discretization error $\lambda T / K$ are balanced to $\lambda^{1/3} S^{2/3} (B+R)^{2/3} D_T^{1/3} T^{2/3}$ by setting $K \asymp  \left(\ \frac{\lambda\sqrt T}{S(B+R)\sqrt{D_T}} \right)^{2/3}$. By taking $\delta = 1/ T$ and treating $S$, $B$, $R$, and $\lambda>0$ as fixed constants, we obtain $\mathrm{Reg}_T = \widetilde{O}( d^{1/3} T^{2/3} )$.

\begin{remark}
    A key difference from the known-profile setting is that observations collected at one payment cannot be converted into valid observations at another payment, so each grid payment must be learned only from its own observations. This leads to the additional $\sqrt{K}$ dependence in the statistical error: increasing $K$ improves the approximation over grid, but it increases the number of payments that must be learned separately.
    It also introduces an additional dependence of regret on $\lambda$.
    In the known-profile setting,  the effect of $\lambda$ can be offset by the choice of $K$ without incurring a polynomial statistical cost. In the unknown-profile setting, however, increasing $K$ also increases the statistical error polynomially, and this trade-off leaves the factor $\lambda^{1/3}$ in the regret bound.
\end{remark}

Without knowledge of the noise profile, attaining the same regret rate as in the known-profile setting is no longer possible. The following lower bound shows that the $\widetilde O(T^{2/3})$ regret achieved above is minimax optimal in its dependence on $T$, up to logarithmic factors.

\begin{theorem}[Minimax lower bound in unknown-profile setting]\label{thm:unknown-lower}
    There exist constants $S$, $B$, $R$, $\lambda$ and $C_u > 0$ such that, for all sufficiently large $T$, every randomized online algorithm incurs
    \begin{equation}
        \mathrm{Reg}_T \geq C_u T^{2/3}
    \end{equation}
    on some one-dimensional instance with an unknown noise profile, where Assumptions~\ref{ass:sg} and \ref{ass:degradation} hold.
\end{theorem}

\begin{remark}
    The high-level construction is inspired by the unknown-profile lower bound of \citet{merlis2026online}, but its realization under logistic loss requires an additional step. Unlike the squared-loss setting of \citet{merlis2026online}, the optimized logistic loss has no such closed-form representation. We therefore first establish strict monotonicity and regularity of the scalar Gaussian logistic loss, together with a Lipschitz inverse, and use this inverse to construct noise-variance profiles that realize the desired flat baseline and localized loss perturbations.
\end{remark}

\section{Conclusion}

We studied online logistic regression with paid noisy features, where the learner pays to acquire higher--quality observations. Unlike the paid linear regression with squared loss, we showed that second-order information alone is insufficient to characterize the effect of measurement noise: even equal noise variance can yield different optimized logistic losses, and covariance ordering may not imply the loss ordering (Proposition~\ref{prop:insufficient}). To recover the monotonicity of prediction loss in payment, we introduced an ordered additive degradation condition (Assumption~\ref{ass:degradation}) under which higher payments returns more informative observations.

Under this structure, we characterized the effect of knowing the noise profile. When the distributions of degradation increments are known, observations collected at higher payments can be exploited to simulate observations at lower payment levels (Proposition~\ref{prop:simulated-concentration}), enabling a simulation-based elimination algorithm to achieve $\widetilde{O}(\sqrt{dT})$ regret (Theorem~\ref{thm:known-upper}). When the profile is unknown, such sample simulation is unavailable and each payment level must be learned separately, leading to $\widetilde{O}(T^{2/3})$ regret (Theorem~\ref{thm:unknown-upper}). Matching lower bounds (Theorems~\ref{thm:known-lower} and \ref{thm:unknown-lower}) show that they are optimal in dependence on horizon $T$ up to logarithmic factor.

Our results rely on the ordered additive degradation assumption (Assumption~\ref{ass:degradation}), which provides a tractable ordering of information quality but is generally stronger than covariance ordering \citep{merlis2026online} . An imporatant direction of future work is to relax these structure, for example by considering weaker conditions or partially known noise profiles. Other natural extensions include generalized linear models, costly information acquisition in multi-armed bandits, strategic bandits, active learning, and preference learning.

\bibliography{iclr2027_conference}
\bibliographystyle{iclr2027_conference}
\newpage
\appendix
\tableofcontents

\newpage
\section*{Appendix}

\section{Related Work}\label{app:related-work}

\paragraph{Errors-in-variables models.}
Errors-in-variables (EIV) models have a long history in statistics. Much of the modern theory for linear EIV focuses on parameter estimation in high-dimensional settings under structural assumptions such as sparsity or low rank \citep{rosenbaum2010sparse,loh2011high,chen2013noisy,rosenbaum2013improved,kaul2015weighted,belloni2017linear,datta2017cocolasso}; see \citet{fuller2009measurement,carroll2006measurement,li2024update} for general treatments and recent reviews. Related work has also considered adaptively selected covariates and prediction with corrupted or missing designs \citep{agarwal2019robustness,agarwal2023adaptive}.

More directly related to our setting is the literature on nonlinear EIV models, and in particular binary and logistic regression with mismeasured covariates. Early work showed that covariate measurement error can substantially distort binary-response models and lead to biased logistic-regression estimates \citep{carroll1984errors,stefanski1985covariate}. This motivated conditional- and corrected-score methods for logistic and, more generally, generalized linear measurement-error models \citep{stefanski1987conditional,nakamura1990corrected}, as well as likelihood, calibration, and semiparametric approaches \citep{rosner1989correction,carroll1991semiparametric,kuha1994corrections,wang1997semiparametric}. Later work developed nonparametric and approximate corrected-score methods \citep{rabehesketh2003correcting,buzas2009note}, and extended measurement-error corrections to high-dimensional generalized linear and logistic models \citep{sorensen2018covariate,park2019logistic}. More recent work studies logistic models with instrumental variables, correlated measurement errors, or high-dimensional noisy and missing covariates \citep{song2019gmm,cao2023logistic,verchand2024high}. These works primarily aim to recover or perform inference on the regression relationship defined by the latent clean covariates. Our objective is different: prediction is performed using the corrupted feature itself, whose quality is controlled through a payment, and hence the optimal logistic-linear predictor need not coincide with the parameter generating the clean responses.

\paragraph{Online learning with noisy and costly information.}
Noisy or corrupted inputs have also been studied in online learning, including online prediction with randomly corrupted examples \citep{cesabianchi2011online} and contextual bandits with noisy or privatized contexts \citep{zheng2020locally,kim2023contextual}. A particularly relevant starting point for the explicit cost--information trade-off is \citet{vanderhoeven2023trading}, who study online binary classification with paid stochastic experts, where larger payments improve the reliability of the observed experts. Their setting differs from ours in that information is provided by finitely many binary experts with independently corrupted predictions, whereas we observe a continuous noisy feature vector and perform logistic-linear prediction.

The closest predecessor of our work is \citet{merlis2026online}, who introduced online linear regression with paid stochastic features. Their model studies essentially the same interaction between feature quality and acquisition cost under squared loss, and establishes optimal horizon dependences of $\widetilde{\Theta}(\sqrt{T})$ and $\widetilde{\Theta}(T^{2/3})$ when the payment--noise profile is known and unknown, respectively. Moving from squared to logistic loss, however, fundamentally changes how feature noise affects prediction: in the linear case its effect is characterized by the noise covariance, whereas for logistic loss covariance alone is insufficient and even covariance ordering need not preserve risk ordering. This motivates the stronger distributional degradation structure used in our analysis and leads to different information-sharing arguments, despite the same optimal dependence on the horizon.

\paragraph{Privacy, incentives, and data acquisition.}
The possibility of paying for more informative data is closely connected to mechanism design for privacy-sensitive data acquisition. A substantial literature studies how an analyst should compensate individuals for privacy loss while achieving a desired estimation accuracy or respecting a fixed acquisition budget \citep{ghosh2011selling,nissim2012privacy,nissim2014redrawing,wang2018value,hsu2014differential,fallah2024optimal}. In these models, agents may have heterogeneous and privately known valuations for privacy, so the central issue is often incentive compatibility and the joint design of payments and privacy guarantees. Recent work has also considered sequential mechanisms for acquiring differentially private data specifically for training logistic regression models \citep{anjarlekar2024online}. Our setting abstracts away strategic reporting and mechanism-design constraints: the payment deterministically indexes a noise distribution, and the learner instead faces the statistical problem of choosing the information quality that optimally balances predictive loss and acquisition cost at each round.

\paragraph{Cost-sensitive and active feature acquisition.}
Our formulation is also related to the literature on costly or active feature acquisition, where a learner decides which information to obtain while trading predictive performance against acquisition cost. Early work on active feature-value acquisition selects missing feature values according to their expected value for improving a predictive model \citep{saar2009active}. More recent approaches formulate classification with costly features as a sequential decision problem and learn instance-dependent acquisition policies, for example using reinforcement learning \citep{janisch2019classification}. These models generally decide which coordinates or tests to acquire, rather than how accurately a common latent feature vector should be measured. Moreover, much of this literature focuses on batch prediction or learned acquisition policies, whereas our objective is an online regret guarantee relative to the optimal fixed payment--predictor pair. Nevertheless, both settings capture the same broad principle that predictive information has an endogenous acquisition cost.

\paragraph{Online logistic regression and generalized linear prediction.}
Finally, our work builds on the literature on online logistic regression and sequential prediction with log loss. Relative-loss guarantees for exponential-family prediction already appear in early work on online density estimation \citep{azoury2001relative}, while logarithmic regret for suitable curved losses follows from second-order methods for online convex optimization \citep{hazan2007logarithmic}. For logistic regression, \citet{hazan2014logistic} studied tight stochastic and online optimization bounds and highlighted the unfavorable dependence of proper algorithms on the radius of the parameter set. \citet{foster2018logistic} showed that improper prediction can substantially improve this dependence by exploiting the mixability of logistic loss, and \citet{shamir2020logistic} further characterized minimax logistic-regression regret across different norm constraints and dimensional regimes. Logistic models have also been extensively studied in the bandit setting, beginning with generalized linear bandits \citep{filippi2010parametric} and followed by work identifying logistic-specific curvature and instance-dependent phenomena \citep{faury2020improved,abeille2021instance}. These works assume that the covariates entering the logistic predictor are observed at a fixed quality. Our setting introduces an additional online decision---how much to pay for the quality of those covariates---and shows that this acquisition decision fundamentally changes the attainable dependence on the horizon.

\section{Proof of Eq. \eqref{eq:regret}}\label{app:regret}
Since $c_t$ and $\nu_t$ are chosen before observing $(\widehat X_t,Y_t)$,
\begin{equation}
    \E \left[ \ell ( Y_t, \widehat{X}_t^\top \nu_t ) + \lambda c_t \middle| \gH_{t-1}, c_t, \nu_t \right] = L(c_t, \nu_t).
\end{equation}
Hence, by the tower property,
\begin{equation}
    \E \left[ \ell ( Y_t, \widehat{X}_t^\top \nu_t ) + \lambda c_t  \right] = \E \left[ L(c_t, \nu_t) \right],
\end{equation}
and therefore, 
\begin{equation}
    \mathrm{Reg}_T = \E \left[ \sum_{t=1}^{T} \big( \ell (Y_t, \widehat{X}_t^\top \nu_t \big) + \lambda c_t \big) \right] - T L^\star = \E \left[ \sum_{t=1}^{T} \{ L(c_t, \nu_t) - L^\star \} \right].
\end{equation}
\newpage
\section{Proof of Proposition~\ref{prop:insufficient}}\label{app:insufficient}

\begin{proof}
Let $d=1$, $X=1$, and $0 < \theta^\star< S$, then $Y \sim \mathrm{Bernoulli} (\mu(\theta^\star))$.
For any $z \in \sR$, 
\begin{equation}
    \E_Y \big[ \ell(Y, z) \big] = \E_Y \left[ \log \big(1 + \exp(z) \big) -Yz\right] = \log \big( 1 + \exp(z) \big) - \mu(\theta^\star) z.
\end{equation}
Note that $\E[\ell(Y,z)]$ is convex and $\theta^\star$ is its unique minimizer.
Its direct results are
\begin{equation}
    \E_Y \big[ \ell(Y, \theta^\star) \big] < \E_Y \big[ \ell(Y, 0) \big] = \log 2,
\end{equation}
and for $z \leq0 $
\begin{equation}
    \E_Y \big[ \ell(Y, z) \big] \geq \E_Y \big[ \ell(Y, 0) \big] = \log 2.
\end{equation}

Without loss of generality, take $c_1=0$ and $c_2 = 1$, and for $a \geq \sqrt{2}$ and $b \geq 1$, define
\begin{equation}
    N(c_1) = 
    \begin{cases}
        a, &\text{w.p. $1/a^2$}\\
        -a, &\text{w.p. $1/a^2$}\\
        0, &\text{w.p. $1 - 2/a^2$}
    \end{cases}
\end{equation}
and
\begin{equation}
    N(c_2) = 
    \begin{cases}
        b, &\text{w.p. $1/2$}\\
        -b, &\text{w.p. $1/2$}
    \end{cases}.
\end{equation}
Both noises have mean zero with $\mathrm{Var}(N(c_1))=2$ and $\mathrm{Var}(N(c_2))=b^2$.

For $N(c_1)$, since $\theta^\star \in \gV = [-S, S]$, 
\begin{align}
    \min_{\nu \in \gV} \E \left[ \ell \big( Y, \widehat{X}(c_1) \nu \big) \right] & \leq \E \left[ \ell \big( Y, \widehat{X}(c_1) \theta^\star \big) \right]\\
    & = \E \left[ \ell \big( Y, \big(1 + N(c_1) \big) \theta^\star \big) \right]\\
    & = \left( 1 - \frac{2}{a^2} \right) \E \left[ \ell \big( Y, \theta^\star \big) \right] + \frac{1}{a^2} \E \left[ \ell \big( Y, (1 + a ) \theta^\star \big) \right]\\
    & \qquad + \frac{1}{a^2} \E \left[ \ell \big( Y, (1 - a ) \theta^\star \big) \right].
\end{align}
Since $\ell(y,z)$ is $1$-Lipschitz in $z$, 
\begin{equation}
    \E \left[ \ell \big( Y, (1 \pm a ) \theta^\star \big) \right] \leq \E \left[ \ell \big( Y, \theta^\star \big) \right] + a \theta^\star.
\end{equation}
Therefore,
\begin{equation}
    \min_{\nu \in \gV} \E \left[ \ell \big( Y, \widehat{X}(c_1) \nu \big) \right] \leq \E \left[ \ell \big( Y, \theta^\star \big) \right] + \frac{2\theta^\star}{a}.
\end{equation}

For $N(c_2)$, for every $\nu \in \gV$,
\begin{align}
    \E \left[ \ell \big( Y, \widehat{X}(c_2) \nu \big) \right] = \frac{1}{2} \E \left[ \ell \big( Y, (1+b) \nu \big) \right] + \frac{1}{2} \E \left[ \ell \big( Y, (1-b) \nu \big) \right].
\end{align}
For any $\nu$
Since $b \geq 1$, one of $(1+b)\nu$ and $(1-b)\nu$ is nonpositive.
One is larger than $\log 2$, and the other is larger than $\E[\ell(Y, \theta^\star)]$.
Therefore, by taking minimum over $\nu\in \gV$,
\begin{equation}
    \min_{\nu \in \gV} \E \left[ \ell \big( Y, \widehat{X}(c_2) \nu \big) \right] \geq \frac{1}{2} \big( \log 2 + \E \left[ \ell (Y, \theta^\star ) \right] \big).
\end{equation}

Now choose $a$ sufficiently large so that
\begin{equation}
    a > \frac{4 \theta^\star}{\log2 - \E \big[ \ell(Y, \theta^\star) \big]},
\end{equation}
then
\begin{equation}
    \min_{\nu \in \gV} \E \left[ \ell \big( Y, \widehat{X}(c_1) \nu \big) \right] < \min_{\nu \in \gV} \E \left[ \ell \big( Y, \widehat{X}(c_2) \nu \big) \right].
\end{equation}

Finally, setting $b=\sqrt{2}$ gives $\mathrm{Var}(N(c_1)) = \mathrm{Var}(N(c_2)) = 2 $, while setting $b = 1$ gives $\mathrm{Var}(N(c_1)) = 2 > 1 =\mathrm{Var}(N(c_2))$.
These show that even equal variance can yield different optimized logistic losses, and variance ordering does not imply the logistic loss ordering.
\end{proof}

\newpage
\section{Proof of Proposition~\ref{prop:one-side}}\label{app:one-side}
\begin{proof}
Under Assumption~\ref{ass:degradation}, Eq. \eqref{eq:logistic-loss-ordering} gives, for every $c_1\leq c_2$ and every $\nu \in \gV$
\begin{equation}
    \E \Big[ \ell \big(Y, \big(X + N(c_1) \big)^\top \nu \big) \Big] \geq \E \Big[ \ell \big(Y, \big( X + N(c_2) \big)^\top \nu \big) \Big].
\end{equation}
Thus, for a fixed $\nu \in \gV$
\begin{align}
    L(c_2, \nu) & = \E \Big[ \ell \big(Y, \big( X + N(c_2) \big)^\top \nu \big) \Big] + \lambda c_2\\
    & \leq \E \Big[ \ell \big(Y, \big( X + N(c_1) \big)^\top \nu \big) \Big] + \lambda c_2\\
    & = L(c_1, \nu) + \lambda (c_2 - c_1).
\end{align}
Using this and letting $\nu^\star (c_1) \in \argmin_{\nu \in \gV} L(c_1,\nu)$ and $\nu^\star (c_2) \in \argmin_{\nu \in \gV} L(c_2,\nu)$, we finally obtain
\begin{align}
    L^\star(c_2) & = L \big( c_2, \nu^\star(c_2) \big)\\
    & \leq L \big(c_2, \nu^\star(c_1) \big)\\
    & \leq L \big(c_1, \nu^\star(c_1) \big) + \lambda (c_2 - c_1)\\
    & = L^\star(c_1) + \lambda (c_2 - c_1).
\end{align}
\end{proof}

\newpage
\section{Proof of Proposition~\ref{prop:concentration}}\label{app:concentration}

\begin{proof}

For a random variable $Z$ and $p\ge1$, we denote its $L_p$-norm as
\begin{equation}
    \| Z \|_{L_p} := \Big( \E \big[ |Z |^p \big] \Big)^{1/p}.
\end{equation}
Define a random variable 
\begin{equation}
    Z_t := \sup_{ \nu \in \gV} \left| \frac{1}{t} \sum_{s=1}^{t} \ell \! \left( Y_s, \widehat{X}_s(c)^\top \nu \right) - \E \left[ \ell \left( Y, \widehat{X}(c)^\top \nu \right) \right] \right|.
\end{equation}
Define $\phi(y, z) := \ell(y,z) - \log 2$, then $\phi$ is $1$-Lipschitz in $z$ and $\phi(y,0)=0$.
Therefore, $Z_t$ can be rewritten as:
\begin{equation}
    Z_t = \sup_{ \nu \in \gV} \left| \frac{1}{t} \sum_{s=1}^{t} \phi \! \left( Y_s, \widehat{X}_s(c)^\top \nu \right) - \E \left[ \phi \left( Y, \widehat{X}(c)^\top \nu \right) \right] \right|.
\end{equation}

Letting $W_s = (Y_s, \widehat{X}_s (c))$, $\gF = \{f_\nu: \nu \in \gV\}$, where $f_\nu(y, x) = \phi (y, x^\top \nu)$, and $\Phi(x) = x^p$ with $p \geq 2$, Lemma~\ref{lem:symmetrize-bound} gives
\begin{equation}
    \| Z_t \|_{L_p} \leq 2 \left\| \sup_{\nu \in \gV} \left| \frac{1}{t} \sum_{s=1}^{t} \varepsilon_s \phi \left(Y_s, \widehat{X}_s(c)^\top \nu \right) \right| \right\|_{L_p},
\end{equation}
where $\varepsilon_s$'s are independent Rademacher random variables.

Now, conditional on observations $\gD_t := \{ ( \widehat{X}_s(c), Y_s)\}_{s=1}^{t}$, let $f_s(z) := \phi (Y_s , z)$ and $A := \{ (\widehat{X}_1(c)^\top \nu , \cdots, \widehat{X}_t(c)^\top \nu) : \nu \in \gV \}$.
Since $\phi(y,z)$ is $1$-Lipschitz in $z$ and $f_s(0) =\phi(Y_s, 0)=0$, $f_s$ is a contraction. Moreover, conditional on $\gD_t$, $A$ is bounded because $|\widehat{X}_s(c)^\top \nu | \leq S \| \widehat{X}_s(c) \|_2 < \infty$ for finite $s$. By taking $\Phi(x)=x^p$, we can apply Lemma~\ref{lem:contraction} as
\begin{align}
    \E_\varepsilon \left[ \left( \frac{1}{2} \sup_{\nu \in \gV} \left| \frac{1}{t} \sum_{s=1}^{t} \varepsilon_s \phi \! \left(Y_s, \widehat{X}_s(c)^\top \nu \right) \right| \right)^p \middle| \gD_t \right]
    & \leq 
    \E_\varepsilon \left[ \left( \sup_{\nu \in \gV} \left| \frac{1}{t} \sum_{s=1}^{t} \varepsilon_s \widehat{X}_s(c)^\top \nu \right|\right)^p \middle| \gD_t \right]\\
    & =
    \E_\varepsilon \left[ \left( \frac{S}{t} \left\| \sum_{s=1}^{t} \varepsilon_s \widehat{X}_s(c) \right\|_2 \right)^p \middle| \gD_t \right],
\end{align}
where the equation holds because $\sup_{\nu \in \gV} | v ^\top \nu| = S \| v \|_2$.
Taking expectation over $\gD_t$, followed by taking the $p$-th root, gives
\begin{equation}
    \frac{1}{2} \left\| \sup_{\nu \in \gV} \left| \frac{1}{t} \sum_{s=1}^{t} \varepsilon_s \phi \! \left(Y_s, \widehat{X}_s(c)^\top \nu \right) \right| \right\|_{L_p} \leq \frac{S}{t} \left( \E \left[ \left\| \sum_{s=1}^{t} \varepsilon_s \widehat{X}_s(c) \right\|_2^p \right] \right)^{1/p}
\end{equation}
Combining the above results, we finally get
\begin{align}
    \| Z_t \|_{L_p} & \leq \frac{4S}{t} \left( \E \left[ \left\| \sum_{s=1}^{t} \varepsilon_s \widehat{X}_s(c) \right\|_2^p \right] \right)^{1/p}.
\end{align}

Define
\begin{equation}
    V_t :=\sum_{s=1}^{t} \varepsilon_s \widehat{X}_s(c),
\end{equation}
then $\E[V_t] = 0$. To apply Lemma~\ref{lem:p-moment-bound}, we first show that $V_t$ is sub-Gaussian random vector.
Let $m := \E [X]$, then the observed feature can be decomposed as $\widehat{X}(c) = m + \big( \widehat{X}(c) - m \big)$. By Assumption~\ref{ass:sg}, $\|m \|_2 \leq B$, and $\big( \widehat{X}(c) - m \big) =(X-m) +N(c) $ is $2R^2$-sub-Gaussian due to independence of $(X-m)$ and $N(c)$. Since $\varepsilon$ has values $+1$ or $-1$ with equal probability,
\begin{align}
    \E_{\varepsilon, X, N(c)} \left[ \exp \left( u^\top \big( \varepsilon \widehat{X}(c) \big) \right) \right] & = \frac{1}{2} \E_{X, N(c)} \left[ \exp \left( u^\top \widehat{X}(c) \right) \right] + \frac{1}{2} \E_{X, N(c)} \left[ \exp \left( - u^\top \widehat{X}(c) \right) \right].
\end{align}
For the first term,
\begin{align}
    \E_{X, N(c)} \left[ \exp \left( u^\top \widehat{X}(c) \right) \right] & = \E_{X, N(c)} \left[ \exp \left( u^\top \Big( \big( \widehat{X}(c) - m \big) + m \Big) \right) \right]\\
    & \leq \exp ( u^\top m ) \exp \big( R^2 \| u \|_2^2 \big),
\end{align}
and similarly,
\begin{equation}
    \E_{X, N(c)} \left[ \exp \left( - u^\top \widehat{X}(c) \right) \right] \leq \exp \big( - u^\top m \big) \exp \big( R^2 \| u \|_2^2 \big).
\end{equation}
Therefore, since $\cosh(z) \leq \exp (z^2/2)$ and $\| m \|_2 \leq B$, we finally get
\begin{align}
    \E_{\varepsilon, X, N(c)} \left[ \exp \left( u^\top \big( \varepsilon \widehat{X}(c) \big) \right) \right] & \leq \frac{1}{2} \left( \exp \big( u^\top m \big) + \exp \big( - u^\top m \big) \right) \exp \big( R^2 \| u \|_2^2 \big)\\
    & = \cosh \big( u^\top m \big) \exp \big( R^2 \| u \|_2^2 \big)\\
    & \leq \exp \left( R^2 \| u \|_2^2 + \frac{\big( u^\top m\big)^2}{2}\right) \\
    & \leq \exp \left( \frac{(B^2 + 2R^2)  \| u \|_2^2 }{2}\right),
\end{align}
that is, $\varepsilon \widehat{X}(c)$ is $(B^2 + 2R^2)$-sub-Gaussian.
Since the pairs $\big(\varepsilon_s, \widehat{X}_s(c) \big)$, $s= 1, \cdots, t$, are independent across $s$, it follows that $V_t$ is $t(B^2+2R^2)$-sub-Gaussian. Therefore Lemma~\ref{lem:p-moment-bound} implies, for every $p\geq 2$ 
\begin{align}
    \left( \E \left[ \left\| \sum_{s=1}^{t} \varepsilon_s \widehat{X}_s(c) \right\|_2^p \right] \right)^{1/p} & \leq C \sqrt{  t(B^2 + 2R^2) (d + p)}\\
    & \leq \sqrt{2} C (B+R) \sqrt{t(d+p)},
\end{align}
where the second inequality is from $\sqrt{B^2 + 2R^2} \leq B + \sqrt{2} R \leq \sqrt{2} (B+R)$. Absorbing all constants into the universal constant $C$ and combining this with the above result, we obtain
\begin{equation}\label{eq:lp-bound}
    \| Z_t \|_{L_p} \leq C S (B+R) \sqrt{\frac{d+p}{t}}.
\end{equation}
Now, choose $p =  2 + \log \frac{1}{\delta}$ so that $p \geq 2$, then by Markov's inequality,
\begin{align}
    \mathbb{P} \left(Z_t>e\|Z_t\|_{L_p}\right) & = \mathbb{P} \left(Z_t^p > e^p \|Z_t\|_{L_p}^p\right)\\
    & \leq \frac{\E [ Z_t^p]}{e^p \|Z_t\|_{L_p}^p}\\
    & = e^{-p} = \frac{\delta}{e^2}
    \le\delta.
\end{align}
Moreover, $d + p = d + 2 + \log \frac{1}{\delta} \leq 3 ( d + \log\frac{1}{\delta} )$ for $d\geq1$ and $\delta\in (0,1)$. Therefore, by Eq. \eqref{eq:lp-bound}, 
\begin{equation}
    e \| Z_t \|_{L_p} \leq C_0 S (B+R) \sqrt{\frac{d + \log (1/\delta)}{t}}
\end{equation}
for another universal constant $C_0 > 0$.
This implies
\begin{equation}
    \mathbb{P} \left( Z_t > C_0 S (B + R) \sqrt{\frac{d + \log(1/\delta)}{t}} \right) \leq \delta,
\end{equation}
and together with the fact
\begin{equation}
    Z_t = \sup_{\nu\in\mathcal V}
        \left|\frac{1}{t}\sum_{s=1}^{t} \ell \! \left( Y_s,\widehat X_s(c)^\top\nu\right)
        +\lambda c-L(c,\nu)\right|,
\end{equation}
we get the desired result.

\end{proof}

\subsection{Supporting Lemmas}

\begin{lemma}[Proposition 4.11 of~\cite{wainwright2019high}]\label{lem:symmetrize-bound}
    Let $W_1, \cdots, W_n$ be i.i.d. copies of random element $W$ in a measurable space $\gW$, and let $\gF$ be a class of measurable functions $f: \gW \to \sR$. Let $\varepsilon_1, \cdots, \varepsilon_n$ be independent Rademacher random variables independent $W_1, \cdots, W_n$. Then, for any convex non-decreasing function $\Phi: \sR_+ \to \sR_+$, 
    \begin{equation}
        \E_{W} \left[ \Phi\left( \sup_{f \in \gF} \left| \frac{1}{n}\sum_{s=1}^{n}f(W_s) - \E \big[ f(W) \big] \right| \right)\right] \leq \E_{W, \varepsilon} \left[ \Phi \left( 2 \sup_{f \in\gF} \left| \frac{1}{n}\sum_{s=1}^{n} \varepsilon_s f(W_s) \right| \right) \right].
    \end{equation}
\end{lemma}

\begin{lemma}[Theorem 4.12 of~\cite{ledoux1991probability}]\label{lem:contraction}
        Let $A \subset \sR^n$ be bounded, and let $f_s:\sR \to \sR$, $s=1, \cdots, n$, satisfy $f_s(0)=0$ and $|f_s (x) - f_s(y) | \leq |x-y|$ for all $x,y \in \sR$. Then, for any convex non-decreasing function $\Phi: \sR_+ \to \sR_+$,
    \begin{equation}
        \E_{\varepsilon} \left[ \Phi \left( \frac{1}{2} \sup_{(x_1, \cdots, x_n) \in A} \left| \frac{1}{n} \sum_{s=1}^{n} \varepsilon_s f_s(x_s) \right| \right) \right] \leq \E_{\varepsilon} \left[ \Phi \left(  \sup_{(x_1, \cdots, x_n) \in A} \left| \frac{1}{n} \sum_{s=1}^{n} \varepsilon_s x_s \right| \right) \right].
    \end{equation}
\end{lemma}


\begin{lemma}[Theorem 32 of \cite{tan2018polynomial}]\label{lem:p-moment-bound}
    Let $V \in \sR^d$ be mean-zero $\sigma^2$-sub-Gaussian random vector satisfying
    \begin{equation}
        \E \left[ \exp \big( u^\top V \big) \right] \leq \exp \left( \frac{\sigma^2 \| u \|_2^2}{2} \right)
    \end{equation}
    for all $u \in \sR^d$. Then for every $p \geq 2$, there exists a universal constant $C > 0$ such that
    \begin{equation}
        \left( \E \big[ \| V \|_2^p \big] \right)^{1/p} \leq  C \sigma \sqrt{d + p}.
    \end{equation}
\end{lemma}

\newpage
\section{Proof of Proposition~\ref{prop:simulated-concentration}}\label{app:simulated-concentration}

\begin{proof}
Fix a target payment $c\in[0,1]$ and a sample count $t \geq 1$. For every $s<\tau(c)$, we have $c_s \geq c$, so the learner can generate simulated feature 
\begin{equation}
    \widetilde{X}_s(c) = \widehat{X}_s + \xi^{(s)}_{c,c_s}.
\end{equation}
Conditional on $\gH_{s-1}$ and $c_s$, the increment $\xi^{(s)}_{c,c_s}$ is independent of $(\widehat{X}_s, Y_s)$ and by Assumption~\ref{ass:degradation}, 
\begin{equation}
    (\widetilde{X}_s(c), Y_s) \overset{d}{=} ( \widehat{X}(c), Y).
\end{equation}
Moreover, the conditional distribution of $(\widetilde{X}_s(c), Y_s)$ is the same fixed distribution regardless of the history. Therefore, for $s < \tau(c)$, the pair $(\widetilde{X}_s(c), Y_s)$ is independent of the preceding history and has the same distribution as $(\widehat{X}(c), Y)$.

To apply Proposition~\ref{prop:concentration}, we complete the sequence to length $t$ with virtual observations from $\tau(c)$, only for the purpose of proof. For $s < \tau(c)$, set
\begin{equation}
    \big( \overline{X}_s(c), {Y}_s \big) = \big( \widetilde{X}_s(c), Y_s \big),
\end{equation}
and for $s \geq \tau(c)$, generate a virtual observation $(\overline{X}_s(c), {Y}_s )$ independently from the distribution of $(\widehat{X}(c), Y)$, and independently of preceding observations as well.

For every $s=1,\cdots, t$, the completed pair $\overline{X}_s(c), Y_s)$ has the same distribution to $(\widehat{X}(c), Y)$ and is independent of the previous observations. If $s < \tau(c)$, it follows from the equivalence of conditional distributions above, while if $s \geq \tau(c)$, it holds by construction.
Therefore, the complete sequence $\big( (\overline{X}_s(c), {Y}_s) \big)_{s=1}^{t}$ is an i.i.d. sequence with the same distribution as $(\widehat{X}(c), Y)$.

Now, we can apply Proposition~\ref{prop:concentration} as
\begin{equation}
    \mathbb{P} \left( \sup_{\nu \in \gV} \left| \frac{1}{t} \sum_{s=1}^{t} \ell \! \left( {Y}_s, \overline{X}_s (c)^\top \nu \right) +\lambda c - L(c, \nu) \right| > \gamma ( t ,\delta) \right) \leq \delta.
\end{equation}
On the event $t < \tau(c)$, there is no virtual observation in the sequence, i.e. $\overline{X}_s(c) = \widetilde{X}_s(c)$ for all $s = 1, \cdots, t$. Consequently,
\begin{align}
    & \left\{ t < \tau(c), \quad \sup_{\nu \in \gV} \left| \frac{1}{t} \sum_{s=1}^{t} \ell \! \left( {Y}_s, \widetilde{X}_s (c)^\top \nu \right) +\lambda c - L(c, \nu) \right| > \gamma ( t ,\delta) \right\} \\
    & \qquad = \left\{ t < \tau(c), \quad \sup_{\nu \in \gV} \left| \frac{1}{t} \sum_{s=1}^{t} \ell \! \left( {Y}_s, \overline{X}_s (c)^\top \nu \right) +\lambda c - L(c, \nu) \right| > \gamma ( t ,\delta) \right\}\\
    & \qquad \subseteq \left\{ \sup_{\nu \in \gV} \left| \frac{1}{t} \sum_{s=1}^{t} \ell \! \left( {Y}_s, \overline{X}_s (c)^\top \nu \right) +\lambda c - L(c, \nu) \right| > \gamma ( t ,\delta) \right\}.
\end{align}
Combining this inclusion with the concentration result above, we obtain
\begin{equation}
    \mathbb{P} \left(  t < \tau(c), \quad \sup_{\nu \in \gV} \left| \frac{1}{t} \sum_{s=1}^{t} \ell \! \left( {Y}_s, \widetilde{X}_s (c)^\top \nu \right) +\lambda c - L(c, \nu) \right| > \gamma ( t ,\delta) \right) \leq \delta.
\end{equation}
Equivalently, with probability at least $1-\delta$,
\begin{equation}
    t < \tau(c) \quad \Longrightarrow \quad  \sup_{\nu \in \gV} \left| \frac{1}{t} \sum_{s=1}^{t}\ell \! \left(Y_s,\widetilde{X}_s(c)^\top\nu\right)
        +\lambda c-L(c,\nu)\right|  \leq \gamma(t, \delta)
\end{equation}
since $\{ A \Rightarrow B\}^c = A \cap B^c$.

\end{proof}

\newpage
\section{Proof of Theorem~\ref{thm:known-upper}}
\label{app:known-upper}

\begin{proof}

Let $\gN_T = \{1,2,4, 8, \cdots \} \cap [T-1]$ denote the set of checkpoints, which is the sample count afters immediately after a batch ends and empirical estimates are computed, but before elimination is performed, then $| \gN_T| = \lceil \log_2 T \rceil$.
For every grid index $k \in [K]_0$ and checkpoint $n \in \gN_T$, define the failure event as
\begin{equation}
    \gF_{k,n} : = \left\{ n < \tau(c^{(k)}), \quad \sup_{\nu \in \gV} \left| \widehat{L}_{k,n}(\nu) - L(c^{(k)}, \nu) \right| > \gamma(n, \delta_T) \right\},
\end{equation}
that is, on $\gF_{k,n}$, the first $n$ payments are all at least $c^{(k)}$, so $n$ valid (simulated) observations at payment $c^{(k)}$ are available, but the concentration bound fails.
By Proposition~\ref{prop:simulated-concentration}, for every $k$ and $n$,
\begin{equation}
    \mathbb{P} ( \gF_{k,n}) \leq \delta_T.
\end{equation}

For  analysis, we keep generating simulated observations at payment $c^{(k)}$ as long as $s < \tau(c^{(k)})$ regardless of whether $c^{(k)}$ is active at checkpoint $n$ or not.
This is only a proof-side construction and does not affect the algorithm.

Define the good event over the whole procedure as
\begin{equation}
    \gG := \bigcap_{k\in [K]_0} \bigcap_{n \in \gN_T} \gF_{k,n}^c,
\end{equation}
then by union bound,  
\begin{equation}
    \mathbb{P}\left( \gG^c\right) = \mathbb{P} \left( \bigcup_{k\in[K]_0} \bigcup_{n \in \gN_T} \gF_{k,n} \right) \leq \sum_{k \in [K]_0} \sum_{n \in \gN_T} \mathbb{P} \left( \gF_{k,n} \right) \leq \delta.
\end{equation}
From now on, we work on the good event $\gG$.

Suppose a payment $c^{(k)}$ is active at checkpoint $n$, then the payment made in the first $n$ rounds are at least $c^{(k)}$ by the selection rule.
Therefore, if $c^{(k)}$ is active at checkpoint $n$, then $n < \tau(c^{(k)})$, which implies on $\gG$, 
\begin{equation}\label{eq:checkpoint-concentration}
    \sup_{\nu \in \gV} \left| \widehat{L}_{k,n} ( \nu) - L(c^{(k)}, \nu) \right| \leq \gamma (n, \delta_T).
\end{equation}

For simplicity, write $\gamma_n := \gamma(n,\delta_T)$.
Let $\nu^\star_k = \nu^\star(c^{(k)}) \in \argmin_\nu L(c^{(k)},\nu)$ so that $L^\star(c^{(k)}) = L(c^{(k)}, \nu^\star_k)$, and recall that
\begin{equation}
    \widehat{\nu}_{k,n} \in \argmin_\nu \widehat{L}_{k,n}(\nu), \qquad
    \widehat{L}_{k,n}^\star = \widehat{L}_{k,n}(\widehat{\nu}_{k,n}) \leq \widehat{L}_{k,n}({\nu}_{k}^\star).
\end{equation}
By optimality of $\widehat{\nu}_{k,n}$ and Eq. \eqref{eq:checkpoint-concentration},
\begin{equation}    
    \widehat{L}_{k,n}^\star = \widehat{L}_{k,n}(\widehat{\nu}_{k,n})
    \leq \widehat{L}_{k,n}({\nu}_{k}^\star)
    \leq L(c^{(k)}, \nu_k^\star) + \gamma_n
    = L^\star(c^{(k)}) + \gamma_n,
\end{equation}
and conversely,
\begin{equation}
    \widehat{L}_{k,n}^\star = \widehat{L}_{k,n}(\widehat{\nu}_{k,n}) \geq L(c^{(k)}, \widehat{\nu}_{k,n}) - \gamma_n \geq L^\star(c^{(k)}) - \gamma_n,
\end{equation}
Therefore, for every active payment $c^{(k)}$ at checkpoint $n$,
\begin{equation}\label{eq:optimized-value-concentration}
    \left| \widehat{L}_{k,n}^\star - L^\star(c^{(k)}) \right| \leq \gamma_n,
\end{equation}
that is, for every active payment $c^{(k)}$ at checkpoint $n$, its empirical optimal loss and population optimal loss differ by at most $\gamma_n$.
In addition, again by Eqs. \eqref{eq:checkpoint-concentration} and \eqref{eq:optimized-value-concentration},
\begin{equation}\label{eq:loss-gap}
    L( c^{(k)}, \widehat{\nu}_{k,n})
    \leq \widehat{L}_{k,n}( \widehat{\nu}_{k,n}) + \gamma_n
    \leq \widehat{L}_{k,n}( {\nu}_{k}^\star) + \gamma_n 
    \leq L^\star(c^{(k)}) + 2\gamma_n.
\end{equation}

Let $\gA$ be an active set immediately before elimination at checkpoint $n$ and let $U_n = \min_{k \in \gA} \{ \widehat{L}_{k,n}^\star + \gamma_n \}$ be the threshold, where the index $k$ is eliminated if $\widehat{L}_{k,n}^\star - \gamma_n > U_n$.
Let $k^\star \in \argmin_{k\in[K]_0} L^\star(c^{(k)})$ be the best grid payment index.

We show by induction over checkpoints that $k^\star$ is never eliminated under good event $\gG$.
It is active initially.
Suppose it is active before elimination at checkpoint $n$.
On the good event, by Eq. \eqref{eq:optimized-value-concentration} and the optimality of $k^\star$, for every $k \in \gA$,
\begin{equation}
    \widehat{L}^\star_{k^\star,n} - \gamma_n \leq L^\star(c^{(k^\star)}) \leq L^\star(c^{(k)}) \leq \widehat{L}^\star_{k,n} + \gamma_n.
\end{equation}
Taking minimum of r.h.s. over $k\in\gA$, we get
\begin{equation}
    \widehat{L}^\star_{k^\star,n} - \gamma_n \leq U_n,
\end{equation}
which means that $k^\star$ remains active after checkpoint $n$. Therefore, by induction, $k^\star$ is never eliminated. 

Suppose $k$ remains active after elimination at checkpoint $n$, which means that
\begin{equation}
    \widehat{L}_{k,n}^\star - \gamma_n \leq U_n.
\end{equation}
Since $k^\star$ is active at checkpoint $n$, the definition of $U_n$ implies
\begin{equation}
    U_n \leq \widehat{L}_{k^\star,n}^\star + \gamma_n.
\end{equation}
Combining these two inequalities, 
\begin{equation}
    \widehat{L}_{k,n}^\star \leq \widehat{L}_{k^\star,n}^\star + 2\gamma_n.
\end{equation}
Applying Eq. \eqref{eq:optimized-value-concentration} to both $k$ and $k^\star$, 
\begin{equation}
    L^\star (c^{(k)}) \leq L^\star(c^{(k^\star)}) + 4\gamma_n.
\end{equation}
that is, for every surviving payment $c^{(k)}$, its optimal population loss lies within $4 \gamma_n$ of the best grid payment.
Finally, combining this with Eq. \eqref{eq:loss-gap}, for every surviving payment $c^{(k)}$,
\begin{equation}\label{eq:surviving-bound}
    L(c^{(k)}, \widehat{\nu}_{k,n}) \leq L^\star(c^{(k^\star)}) + 6 \gamma_n.
\end{equation}

Let $\gA'$ be an active set after elimination, and let $k' = \max \gA'$.
During the next batch, the algorithm pays $c^{(k')}$, and uses the predictor $\widehat{\nu}_{k', n}$.
Since $k'$ survives after elimination, Eq. \eqref{eq:surviving-bound} gives
\begin{equation}
    L(c^{(k')}, \widehat{\nu}_{k',n} ) \leq L^\star(c^{(k^\star)}) + 6 \gamma_n.
\end{equation}
For each checkpoint $n \in \gN_T$, define the set of rounds in the batch after checkpoint $n$ as
\begin{equation}
    B_n = \big\{ n+1, n+2, \cdots, \min\{2n, T\} \big\},
\end{equation}
then $|B_n | \leq n$, with equality except possibly for the last batch.
Therefore, cumulative regret during the batch is
\begin{align}
    \sum_{t \in B_n} \big[ L(c_t, \nu_t ) - L^\star(c^{(k^\star)})  \big] = | B_n | \times \big[ L(c^{(k')}, \widehat{\nu}_{k',n} ) - L^\star(c^{(k^\star)}) \big] \leq 6 n \gamma_n.
\end{align}
Since $[T] = \{1 \} \cup \bigcup_{n \in \gN_T} B_n$ with disjoint batches $B_n$, the cumulative regret relative to the best grid payment is decomposed as:
\begin{equation}
    \sum_{t=1}^T \big[ L(c_t, \nu_t ) - L^\star(c^{(k^\star)})  \big] = L(c_1, \nu_1) - L^\star(c^{(k^\star)}) + \sum_{n \in \gN_T} \sum_{t \in B_n} \big[ L(c_t, \nu_t ) - L^\star(c^{(k^\star)})  \big].
\end{equation}

For the first round, the algorithm pays $c_1=1$ and uses $\nu_1=0 $.
Since $\ell(Y,0) = \log 2$ regardless of $Y \in\{0,1\}$, $L(1,0) = \log 2 + \lambda$.
Moreover, since $L^\star(c^{(k^\star)}) \geq 0$, $L(c_1, \nu_1) - L^\star(c^{(k^\star)}) \leq \log 2 + \lambda$.
Combining the above results,
\begin{align}
    \sum_{t=1}^T \big[ L(c_t, \nu_t ) - L^\star(c^{(k^\star)})  \big] & \leq \log 2 + \lambda + \sum_{n \in \gN_T} \sum_{t \in B_n} \big[ L(c_t, \nu_t ) - L^\star(c^{(k^\star)})  \big]\\
    & \leq \log 2 + \lambda + 6\sum_{n \in \gN_T} n \gamma_n.
\end{align}

Recall that
\begin{equation}
    \delta_T = \frac{\delta}{(K+1) \lceil \log_2 T \rceil}, \quad \gamma_n = \gamma(n, \delta_T) = C_0 S (B+R) \sqrt{\frac{d + \log(1/\delta_T)}{n}},
\end{equation}
and $D_T = d + \log (1/ \delta_T)$, then 
\begin{align}
    \sum_{n \in \gN_T} n \gamma_n & = \sum_{n \in \gN_T} C_0 S (B +R) \sqrt{n D_T}\\
    & = C_0 S (B+ R) \sqrt{D_T} \sum_{n \in \gN_T} \sqrt{n}.
\end{align}

Let $p = \lceil \log_2 T \rceil $
Since $\gN_T = \{2^r : r=0,\dots, p-1 \}$, we have
\begin{align}\label{eq:geometric-sum}
    \sum_{n \in \gN_T} \sqrt{n} = \sum_{r=0}^{ p - 1} \sqrt{2^r} = \frac{2^{p/2} - 1}{\sqrt{2} - 1}  \leq \frac{\sqrt{2T}-1}{\sqrt{2}{ - 1}} \leq \frac{\sqrt{2}{}}{\sqrt{2}-1}\sqrt{T}.
\end{align}
Substituting this, for some universal constant $C_1 > 0$,
\begin{equation}
    \sum_{t=1}^{T} \Big[ L(c_t , \nu_t) - L^\star(c^{(k^\star)}) \Big] \leq C_1 S(B+R) \sqrt{TD_T} + \log 2 + \lambda.
\end{equation}

It remains to compare the best grid loss $L^\star(c^{(k^\star)})$ with the optimal loss $L^\star$ over the continuous interval $[0,1]$.
For any $c \in [0,1]$, we compare this with the grid point $c^{(k)} = \lceil Kc \rceil / K$, then
\begin{equation}
    c \leq c^{(k)} \leq c + 1/K.
\end{equation}
Since $c^{(k)} \geq c$, Proposition~\ref{prop:one-side} implies that
\begin{equation}
    L^\star(c^{(k)}) \leq L^\star(c) + \lambda (c^{(k)} - c) \leq L^\star(c) + \frac{\lambda}{K}.
\end{equation}
Since $k^\star$ is the best grid index, $L^\star (c^{(k^\star)}) \leq L^\star(c^{(k)})$ for every $k \in [K]_0$, and thus
\begin{equation}
    L^\star (c^{(k^\star)}) \leq L^\star(c) + \frac{\lambda}{K}
\end{equation}
for any $c \in [0,1]$. By taking infimum over $c \in [0,1]$, we get
\begin{equation}\label{eq:discretize-gap}
    L^\star(c^{(k^\star)}) \leq L^\star + \frac{\lambda}{K}.
\end{equation}

Finally, since this holds on $\gG$, with probability at least $1-\delta$,
\begin{align}
    \sum_{t=1}^{T} \Big[ L(c_t, \nu_t) - L^\star \Big] & = \sum_{t=1}^{T} \Big[ L(c_t , \nu_t) - L^\star(c^{(k^\star)}) \Big] + \sum_{t=1}^{T} \Big[ L^\star(c^{(k^\star)}) - L^\star \Big]\\
    & \leq C_1 S(B+R) \sqrt{TD_T}  + \frac{\lambda T}{K}+ \log 2 + \lambda.
\end{align}

For the expected regret $\mathrm{Reg}_T$, let $R_T = \sum_{t=1}^{T} \big[ L(c_t, \nu_t) - L^\star \big]$, then $\mathrm{Reg}_T = \E [ R_T ]$ by Eq. \eqref{eq:regret}.
We decompose the expectation according to the good event $\gG$ and bad event $\gG^c$. On $\gG$, the previous bound applies. We now bound the regret $R_T$ on $\gG^c$.

For every $y \in \{0, 1\}$ and $z \in \sR$,
\begin{equation}
    \ell(y, z) \leq \log 2 + |z|
\end{equation}
due to its $1$-Lipschitzness in $z$ and the fact that $\ell(y,0)=\log2$.
Let $m := \E[X]$, then for any $c\in[0,1]$ and $\nu \in \gV$,
\begin{align}
    L(c, \nu) & = \E \left[ \ell \big( Y, \widehat{X}(c)^\top \nu \big) \right] + \lambda c\\
    & \leq \log 2 + \E \left[ \big| \widehat{X}(c)^\top \nu \big| \right] + \lambda\\
    & \leq \log 2 + \E \left[ \big| m^\top \nu \big| \right] + \E \left[ \big| (X-m)^\top \nu \big| \right] + \E \left[ \big| N(c)^\top \nu \big| \right] + \lambda\\
    & \leq \log 2 + SB + \sqrt{\E \left[ \big| (X-m)^\top \nu \big|^2 \right]} + \sqrt{\E \left[ \big| N(c)^\top \nu \big|^2 \right]} + \lambda.
\end{align}
Moreover, by Assumption~\ref{ass:sg}, for any fixed $\nu\in\gV$, both $(X-m)^\top \nu$ and $N(c)^\top \nu$ are mean-zero $R^2\|\nu\|_2^2$-sub-Gaussian random variables, hence their variances are bounded by $R^2 \| \nu \|_2^2$. Therefore, together with the fact $L^\star \geq 0$,
\begin{align}
    L(c, \nu) - L^\star & \leq \log 2 +  SB + \sqrt{\E \left[ \big| (X-m)^\top \nu \big|^2 \right]} + \sqrt{\E \left[ \big| N(c)^\top \nu \big|^2 \right]} + \lambda\\
    & \leq \log 2 + SB + 2R \| \nu \|_2 + \lambda \\
    & \leq \log 2 + S (B + 2R) + \lambda.
\end{align}
Consequently, on every realization on $\gG^c$, 
\begin{equation}
    R_T \leq T \big[ \log 2 + S(B+2R)  + \lambda \big].
\end{equation}
Finally, using $\mathbb{P} ( \gG^c) \leq \delta$, we obtain
\begin{align}
    \mathrm{Reg}_T & = \E[R_T \mid \gG] \times \mathbb{P}(\gG)+ \E [R_T \mid \gG^c ] \times \mathbb{P}(\gG^c) \\
    & \leq C_1 S (B+R) \sqrt{T D_T} + \frac{\lambda T }{K} + \log 2 + \lambda\\
    & \qquad +\delta T \big[ \log 2 + S (B+2R) + \lambda \big].
\end{align}
\end{proof}

\newpage
\section{Proof of Theorem~\ref{thm:unknown-upper}}\label{app:unknown-upper}
\begin{proof}

Let $\gM_T := \{1, 2, 4, 8, \cdots \} \cap [T-1]$ be the set of possible local checkpoints, then $| \gM_T| = \lceil \log_2 T \rceil$.   

For each grid $k\in [K]_0$ and its local checkpoint $m \in \gM_T$, define the local good event as
\begin{equation}
    \gG_{k,m} := \left\{ \sup_{\nu \in \gV} \left| \widehat{L}_{k,m}(\nu) - L(c^{(k)}, \nu) \right| \leq \gamma(m, \delta_T) \right\}.
\end{equation}
For analysis, for each grid payment $c^{(k)}$, consider an infinite i.i.d. sequence $(\widehat{X}_{k,s}, Y_{k,s})_{s\geq1}$. When the algorithm selects $c^{(k)}$, it observes the next unobserved pair in order. This makes $\widehat{L}_{k,m}$ well-defined for every $k\in[K]_0$ and $ m \in \gM_T$ even if the algorithm never actually reaches $m$ pulls of $c^{(k)}$.
This is only a proof-side construction and does not affect the algorithm.

Although the rounds at which $c^{(k)}$ is paid are adaptive, the observations $(\widehat{X}_{k,s}, Y_{k,s})_{s\geq 1}$ collected at payment $c^{(k)}$ are i.i.d. samples of $(\widehat{X}(c^{(k)}), Y)$.
Indeed, conditional on history and on selecting $c^{(k)}$, the current observation has this fixed distribution independently of the preceding history.
Therefore, Proposition~\ref{prop:concentration} gives
\begin{equation}
    \mathbb{P} ( \gG_{k,m} ) \geq 1- \delta_T.
\end{equation}

Define a good event as
\begin{equation}
    \gG := \bigcap_{k \in[K]_0} \bigcap_{m \in \gM_T} \gG_{k,m},
\end{equation}
then by union bound,
\begin{equation}
    \mathbb{P} (\gG^c) = \mathbb{P} \left( \bigcup_{k\in[K]_0} \bigcup_{m \in \gM_T} \gG_{k,m}^c \right) \leq \sum_{k\in[K]_0} \sum_{m \in \gM_T} \mathbb{P} \left( \gG_{k,m}^c \right) \leq \delta.
\end{equation}
We work on the good event $\gG$.

For simplicity write $\gamma_m := \gamma(m, \delta_T)$.
As in the known-profile proof (Appendix~\ref{app:known-upper}), uniform concentration implies
\begin{equation}\label{eq:unknown-optimal-concentration}
    \left| \widehat{L}_{k,m}^\star - L^\star(c^{(k)}) \right| \leq \gamma_m
\end{equation}
and
\begin{equation}
    L(c^{(k)}, \widehat{\nu}_{k,m}) \leq \widehat{L}^\star_{k,m} + \gamma_m\leq L^\star(c^{(k)}) + 2\gamma_m.
\end{equation}

Let $k^\star \in \argmin_{k \in[K]_0} L^\star(c^{(k)})$ be the best grid payment index.
At each round $t$ after initialization, i.e., $t \geq K+2$, by the selection rule,
\begin{align}
    \widehat{L}^\star_{k_t, m_t(k_t)} - \gamma_{m_t(k_t)} & \leq \widehat{L}^\star_{k^\star, m_t(k^\star)} - \gamma_{m_t(k^\star)} \\
    & \leq L^\star(c^{(k^\star)}). \tag{by Eq. \eqref{eq:unknown-optimal-concentration}}
\end{align}

By the concentration in the set $\gG_{k_t, m_t(k_t)}$, we know
\begin{equation}
    L(c^{(k_t)}, \widehat{\nu}_{k_t, m_t(k_t)}) \leq \widehat{L}_{k_t, m_t(k_t)}( \widehat{\nu}_{k_t, m_t(k_t)})  + \gamma_{m_t(k_t)} = \widehat{L}^\star_{k_t, m_t(k_t)} + \gamma_{m_t(k_t)}.
\end{equation}
Combining the above inequalities, together with $\nu_t = \widehat{\nu}_{k_t, m_t(k_t)}$
\begin{equation}\label{eq:unknown-decision-upper}
      L(c^{(k_t)}, \nu_t ) \leq L^\star(c^{(k^\star)}) + 2\gamma_{m_t(k_t)}.
\end{equation}

For a fixed grid payment $c^{(k)}$, it is selected $n_{T}(k)$ times over $T$.
Since every payment is selected once in the initialization, we have $n_{T}(k) \geq 1$ almost surely, and since exactly one payment is selected at each round,
\begin{equation}
    \sum_{k \in[K]_0} n_{T}(k) = T.
\end{equation}
For each payment $c^{(k)}$, let $\gM_T(k) := \{ 2^r : r=0, 1, \cdots, \lceil \log_2 n_{T}(k)\rceil - 1 \}$ be the set of local checkpoints after which $c^{(k)}$ is selected at least once.

For each $m\in \gM_T(k)$, let $B_{k,m} := \{ t \geq K+2 : k_t = k, m_t(k)=m\}$ be the set of rounds at which $c^{(k)}$ is selected and estimator computed at checkpoint $m$ is used. This estimator is used until the count doubles, so $|B_{k,m}| \leq m$.

For $t \in B_{k,m}$, $k_t=k$, $m_t(k) = m$, and hence by Eq. \eqref{eq:unknown-decision-upper},
\begin{equation}
    L(c_t, \nu_t) - L^\star(c^{(k^\star)}) \leq 2 \gamma_m.
\end{equation}

Recall that $\gamma_m = \gamma(m, \delta_T) = C_0 S (B+R) \sqrt{D_T/m}$.
By partitioning the rounds after initialization by the sets $B_{k,m}$, 
\begin{align}
    \sum_{t=K+2}^{T} \Big[ L(c_t, \nu_t) - L^\star(c^{(k^\star)})\Big] & = \sum_{k \in[K]_0} \sum_{m \in \gM_T(k)} \sum_{t \in B_{k,m}} \Big[ L (c^{(k_t)}, \nu_t) - L^\star (c^{(k^\star)}) \Big]\\
    & \leq \sum_{k \in[K]_0} \sum_{m \in \gM_T(k)} \sum_{t \in B_{k,m}}2\gamma_{m}\\
    & \leq \sum_{k \in[K]_0} \sum_{m \in \gM_T(k)} 2 m \gamma_{m}\\
    & \leq 2 C_0 S (B+R) \sqrt{D_T} \sum_{k \in[K]_0} \sum_{m \in \gM_T(k)} \sqrt{m}.
\end{align}

By the same geometric-sum argument as in Eq. \eqref{eq:geometric-sum},
\begin{align}
    \sum_{m \in \gM_T(k)} \sqrt{m} \leq \frac{\sqrt{2}}{\sqrt{2} - 1} \sqrt{n_{T}(k)}.
\end{align}
Substituting this into the above decomposition, we obtain
\begin{align}
    \sum_{t=K+2}^{T} \Big[ L(c_t, \nu_t) - L^\star(c^{(k^\star)})\Big] & \leq C_2 S (B+R) \sqrt{D_T} \sum_{k\in[K]_0} \sqrt{n_{T}(k)} \\
    & \leq C_2 S (B+R) \sqrt{ D_T (K+1) \sum_{k\in[K]_0} n_{T}(k)} \tag{Cauchy--Schwarz}\\
    & = C_2 S (B+R) \sqrt{(K+1)T D_T},
\end{align}
where $C_2$ is a universal constant.
It remains to bound the regret during $K+1$ initialization rounds. During initialization, every payment is selected once and the algorithm uses $\nu=0$. Since
\begin{equation}
    L(c^{(k)}, 0) = \log 2 + \lambda c^{(k)} \leq \log 2 + \lambda
\end{equation}
and $L^\star(c^{(k^\star)}) \geq 0$,
\begin{equation}
    \sum_{t=1}^{K+1} \Big[ L(c_t, \nu_t) - L^\star(c^{(k^\star)})\Big] \leq (K+1) (\log 2 + \lambda).
\end{equation}
Combining them,
\begin{equation}
    \sum_{t=1}^{T} \Big[ L(c_t, \nu_t) - L^\star(c^{(k^\star)})\Big] \leq C_2 S (B+R) \sqrt{(K+1)T D_T} + (K+1) (\log 2 + \lambda).
\end{equation}

Finally, using Eq. \eqref{eq:discretize-gap}, 
\begin{align}
    \sum_{t=1}^{T} \Big[ L(c_t, \nu_t) - L^\star \Big] & \leq C_2 S (B+R) \sqrt{(K+1)T D_T} + \frac{\lambda T}{K}\\
    & \qquad+(K+1) (\log 2 + \lambda)
\end{align}
on the event $\gG$, which holds with probability at least $1-\delta$.

By the same argument in Appendix~\ref{app:known-upper}, we get
\begin{align}
    \mathrm{Reg}_T & \leq C_2 S (B+R) \sqrt{(K+1)T D_T} + \frac{\lambda T}{K} +(K+1) (\log 2 + \lambda)\\
    & \qquad + \delta T \big[ \log 2 + S(B+2R) + \lambda \big],
\end{align}
which proves the theorem.
\end{proof}


\newpage
\section{Proof of lower bounds}\label{app:lower}
In this section, we establish the lower bounds for both the known- and unknown-profile settings. Both constructions are carried out in a one-dimensional Gaussian model, which already captures the statistical difficulty responsible for the respective regret rates. We first collect several regularity properties of the optimized scalar Gaussian logistic loss in Section \ref{app:regularity} that will be used throughout the constructions. We then prove the known-profile lower bound in Section \ref{app:known-lower}, followed by the unknown-profile lower bound in Section \ref{app:unknown-lower}.

For an environment $M$ and a possibly randomized policy $\pi$, write
\[
    \Reg_T^\pi(M)
    :=\E_M^\pi\!\left[\sum_{t=1}^T
    \bigl(L_M(c_t,\nu_t)-L_M^\star\bigr)\right]
\]
for expected regret, where $L_M(c,\nu)$ is the population loss including
payment and $L_M^\star=\inf_{c\in[0,1],\,|\nu|\le S}L_M(c,\nu)$.
We index the environment by its clean-feature variance $a$ in the known-profile
construction and by its noise-variance profile $s$ in the unknown-profile
construction. Any additional revelation of $X_t$ occurs after the prediction
and response at round $t$; the payment and predictor are still selected before
observing the current corrupted feature.

\subsection{Regularity of the scalar Gaussian logistic loss}\label{app:regularity}

Fix $S>0$ and $0 < \theta \leq S$. For $a>0$ and $s\ge0$, let
\[
    X\sim\mathcal N(0,a),
    \qquad
    N_s\sim\mathcal N(0,s),
\]
independently, and let $Y\mid X\sim\Ber(\mu(\theta X))$. Define
\[
    m(a):=\E[XY]=\E[X\mu(\theta X)].
\]
By Stein's identity,
\begin{equation}
    m(a)
    =
    a\theta\,\E\!\left[\mu'(\theta\sqrt a Z)\right]
    >0,
    \qquad Z\sim\mathcal N(0,1).
    \label{eq:m-a}
\end{equation}
For $q\ge0$, define
\[
    \Psi(q)
    :=
    \E_{Z\sim\mathcal N(0,1)}
    \left[\log(1+\exp(\sqrt q Z))\right].
\]
The prediction loss excluding payment is defined by
\begin{equation}
    F(a,s,\nu)
    :=
    \Psi((a+s)\nu^2)-m(a)\nu,
    \qquad |\nu|\le S,
    \label{eq:F-def}
\end{equation}
and optimal prediction loss is defined by
\begin{equation}
    H(a,s)
    :=
    \min_{|\nu|\le S}F(a,s,\nu).
    \label{eq:H-def}
\end{equation}

The following lemma provides the basic regularity properties needed by both lower-bound constructions. For the known-profile lower bound, the strict increase of $H(a,s)$ in the noise variance together with differentiability in $a$ will allow us to identify a point at which two payment regions are locally tied and exchange optimality under a small perturbation of the clean-feature variance. For the unknown-profile lower bound, strict monotonicity in $s$ will allow us to prescribe a desired optimized-loss profile and map it back to a Gaussian noise-variance profile.

\begin{lemma}[Regularity of the constrained scalar Gaussian loss]
\label{lem:regularity}
For every $a>0$ and $s\ge0$:
\begin{enumerate}
    \item the map $\nu\mapsto F(a,s,\nu)$ is strictly convex on $[-S,S]$;
    \item the constrained minimizer
    \[
        \nu_\star(a,s)
        :=
        \arg\min_{|\nu|\le S}F(a,s,\nu)
    \]
    is unique and satisfies $0<\nu_\star(a,s)\le S$;
    \item $(a,s)\mapsto\nu_\star(a,s)$ is continuous on $(0,\infty)\times[0,\infty)$;
    \item $H$ is continuously differentiable on $(0,\infty)\times[0,\infty)$, and
    \begin{equation}
        \partial_s H(a,s)
        =
        \nu_\star(a,s)^2
        \Psi'\!\left((a+s)\nu_\star(a,s)^2\right)
        >0.
        \label{eq:H-s-derivative}
    \end{equation}
\end{enumerate}
In particular, the conclusion does not require the constraint $|\nu|\le S$ to be inactive.
\end{lemma}

\begin{proof}
Let $W\sim\mathcal N(0,a+s)$. Then
\[
    \Psi((a+s)\nu^2)
    =
    \E[\log(1+e^{\nu W})].
\]
Differentiating twice in $\nu$ gives
\[
    \partial_{\nu\nu}^2F(a,s,\nu)
    =
    \E\!\left[\mu'(\nu W)W^2\right].
\]
Because $a+s>0$, $W$ is nondegenerate, and $\mu'(z)>0$ for every finite $z$. Hence the right-hand side is strictly positive for every $\nu$, proving strict convexity and uniqueness of the constrained minimizer.

At $\nu=0$,
\[
    \partial_\nu F(a,s,0)=-m(a)<0
\]
by \eqref{eq:m-a}. Thus the objective strictly decreases when moving from zero in the positive direction, so its unique minimizer cannot be zero or negative. Hence
\[
    0<\nu_\star(a,s)\le S.
\]
The boundary value $\nu_\star(a,s)=S$ is allowed.

The function $F$ is jointly continuous in $(a,s,\nu)$, while $[-S,S]$ is compact. Berge's maximum theorem and uniqueness of the minimizer imply continuity of $(a,s)\mapsto\nu_\star(a,s)$.

The map $a\mapsto m(a)$ is continuously differentiable on $(0,\infty)$. This follows from \eqref{eq:m-a} by differentiation under the Gaussian expectation, using boundedness of the relevant derivatives of the logistic function. Moreover, the Gaussian heat equation gives
\[
    \Psi'(q)
    =
    \frac12\E\!\left[\mu'(\sqrt q Z)\right],
    \qquad q\ge0,
\]
with the continuous extension $\Psi'(0)=1/8$. Thus $F$ is continuously differentiable in $(a,s)$.

Since the constrained minimizer is unique, Danskin's theorem yields
\[
    \partial_sH(a,s)
    =
    \partial_sF(a,s,\nu_\star(a,s)),
\]
which is exactly \eqref{eq:H-s-derivative}. Both factors in \eqref{eq:H-s-derivative} are strictly positive. Continuity of the derivative follows from continuity of $\nu_\star$ and of the derivatives of $F$. The same argument applies to the $a$ derivative, proving $H\in C^1$.
\end{proof}

For the unknown-profile construction, we need a quantitative version of this monotonicity: a small perturbation of the optimized logistic loss should induce only a controlled perturbation of the corresponding noise variance. The next corollary provides a uniformly Lipschitz inverse for $s\mapsto H(a,s)$ on a compact variance interval. It will be used both to keep the constructed variance profiles uniformly Lipschitz and to control the KL divergence between nearby profiles.

\begin{corollary}[Uniform inverse regularity]
\label{cor:inverse}
Fix $a>0$ and $0<s_{\min}<s_{\max}<\infty$, and write $H_a(s):=H(a,s)$. Then
\[
    H_a:[s_{\min},s_{\max}]
    \longrightarrow
    [H_a(s_{\min}),H_a(s_{\max})]
\]
is a continuously differentiable, strictly increasing bijection. Furthermore,
\begin{equation}
    g
    :=
    \min_{s\in[s_{\min},s_{\max}]}H_a'(s)
    >0,
    \label{eq:g}
\end{equation}
and
\begin{equation}
    |H_a^{-1}(u)-H_a^{-1}(v)|
    \le
    \frac1g|u-v|.
    \label{eq:inverse-lip}
\end{equation}
\end{corollary}

\begin{proof}
Lemma~\ref{lem:regularity} gives continuity and strict positivity of $H_a'$ on the compact interval. Hence \eqref{eq:g} holds. Strict monotonicity gives invertibility, and the mean-value theorem gives \eqref{eq:inverse-lip}.
\end{proof}

Both lower-bound arguments compare nearby centered Gaussian distributions. We therefore record the following elementary quadratic bound on their KL divergence.

\begin{lemma}[Quadratic KL bound for centered Gaussians]
\label{lem:gaussian-kl}
Let $u,v\ge \underline{s}>0$. Then
\[
    \KL\!\left(\mathcal N(0,u)\,\middle\|\,\mathcal N(0,v)\right)
    \le
    \frac{(u-v)^2}{4\underline{s}^{\,2}}.
\]
\end{lemma}

\begin{proof}
For centered Gaussians,
\[
    \KL\!\left(\mathcal N(0,u)\,\middle\|\,\mathcal N(0,v)\right)
    =
    \frac12\left(\frac uv-1-\log\frac uv\right).
\]
For fixed $v$, the value and first derivative of the right-hand side vanish at $u=v$, while its second derivative with respect to $u$ equals $1/(2u^2)\le 1/(2\underline{s}^{\,2})$. Taylor's theorem therefore gives the claimed bound.
\end{proof}

\subsection{Proof of Theorem \ref{thm:known-lower}: known profile}\label{app:known-lower}

\paragraph{Proof sketch.}
We construct a known step-shaped Gaussian noise profile with two relevant payment regions, $c<c_0$ and $c\ge c_0$. We first choose a clean-feature variance $a_0$ at which the optimized population loss of the two regions are tied. Perturbing the variance to $a_-=a_0-\varepsilon$ and $a_+=a_0+\varepsilon$ makes opposite regions optimal, with a loss gap of order $\varepsilon$ for choosing the wrong region. Even if the clean feature $X_t$ is revealed after every round, the two instances remain statistically hard to distinguish: their transcript KL divergence is of order $T\varepsilon^2$. Taking $\varepsilon\asymp T^{-1/2}$ keeps the two instances close while forcing a linear number of wrong-region decisions in at least one of them, which yields $\Omega(\sqrt T)$ regret.

Fix $s_0>0$ and define
\begin{equation}
    G(a)
    :=
    H(a,s_0)-H(a,0),
    \qquad a>0.
    \label{eq:G}
\end{equation}

\begin{lemma}[Existence of a transverse crossing]
\label{lem:crossing}
There exists $a_0>0$ such that
\[
    G(a_0)>0,
    \qquad
    G'(a_0)>0.
\]
\end{lemma}

\begin{proof}
For every $a>0$, Lemma~\ref{lem:regularity} implies that $s\mapsto H(a,s)$ is strictly increasing. Hence
\[
    G(a)>0
    \qquad\text{for all }a>0.
\]

The function $G$ has a continuous extension to $a=0$. Indeed, $m(a)\to0$ as $a\downarrow0$, and for fixed $s\ge0$ the functions $F(a,s,\nu)$ converge uniformly over $|\nu|\le S$ to
\[
    F(0,s,\nu)=\Psi(s\nu^2).
\]
Since $\Psi$ is nondecreasing and $\nu=0$ is feasible,
\[
    H(0,s)=\min_{|\nu|\le S}\Psi(s\nu^2)=\Psi(0)=\log2
\]
for every $s\ge0$. Therefore
\[
    G(0)=0.
\]

Fix any $A>0$. Since $G(A)>0$ and $G(a)\to0$ as $a\downarrow0$, there exists $a_1\in(0,A)$ such that $G(a_1)<G(A)$. By Lemma~\ref{lem:regularity}, $G$ is continuously differentiable on $(0,\infty)$. The mean-value theorem therefore gives some $a_0\in(a_1,A)$ for which
\[
    G'(a_0)
    =
    \frac{G(A)-G(a_1)}{A-a_1}
    >0.
\]
\end{proof}

\begin{theorem}[Known Gaussian step-profile lower bound]
\label{thm:known-lower-app}
Fix any $S>0$. There exist fixed constants
\[
    \theta\in (0, S],
    \qquad
    a_0>0,
    \qquad
    s_0>0,
    \qquad
    c_0\in(0,1),
    \qquad
    \lambda>0,
\]
a completely known nonincreasing Gaussian variance profile
\begin{equation}
    s(c)=s_0\ind\{c<c_0\},
    \label{eq:step}
\end{equation}
and constants $0 < \rho < a_0$ and $c > 0$ such that, for all sufficiently large $T$ and every possibly randomized learning policy $\pi$,
\begin{equation}
    \inf_\pi
    \sup_{a\in[a_0-\rho,a_0+\rho]}
    \mathrm{Reg}_T^\pi (a)
    \ge
    c\sqrt T.
    \label{eq:known-lb}
\end{equation}
The result remains valid when $\theta$, $S$, $\lambda$, $a_0$, $\rho$ and the intire noise profile are known to the learner, while the variance $a$ remains unknown, and $X_t$ is additionally revealed after every round.
\end{theorem}

\begin{proof}
Choose $0 < \theta \leq S$ and fix any $s_0>0$. Let $a_0$ be supplied by Lemma~\ref{lem:crossing}. Set
\[
    c_0:=\frac12,
    \qquad
    \lambda:=\frac{G(a_0)}{c_0}=2G(a_0)>0.
\]
Consider the known step profile \eqref{eq:step}. Since this Gaussian variance profile is nonincreasing in the payment, it admits the required ordered additive degradation assumption (Assumption~\ref{ass:degradation}): for $c_1\le c_2$, the lower-payment noise can be represented as the higher-payment noise plus an independent centered Gaussian increment. Thus the profile satisfies the ordered additive degradation assumption, and the distributions of the Gaussian increments can be revealed to the learner.

For a fixed clean-feature variance $a$, all payments $c<c_0$ have the same prediction loss $H(a,s_0)$, so the best payment in that region is $c=0$, with optimized population loss $H(a,s_0)$. Likewise, all payments $c\ge c_0$ have prediction loss $H(a,0)$, so the best payment in that region is $c=c_0$, with optimized population loss $H(a,0)+\lambda c_0$.

Define
\begin{equation}
    D(a)
    :=
    H(a,s_0)-H(a,0)-\lambda c_0
    =
    G(a)-G(a_0).
    \label{eq:D}
\end{equation}
Then
\[
    D(a_0)=0,
    \qquad
    D'(a_0)=G'(a_0)>0.
\]
By continuity of $D'$, there exist $\rho,\gamma>0$, with $\rho<a_0$, such that
\begin{equation}
    D'(a)\ge\gamma
    \qquad
    \text{whenever }|a-a_0|\le\rho.
    \label{eq:D-slope}
\end{equation}
All Gaussian instances in this interval satisfy Assumption~\ref{ass:sg}
with any fixed $B>0$ and
\[
    R\ge\max\{1,\sqrt{a_0+\rho},\sqrt{s_0}\}.
\]
These bounds, as well as $\theta,S,\lambda$ and the noise profile, are independent of $T$.

Fix any possibly randomized policy $\pi$.
For $0<\varepsilon\le\rho$, define
\[
    a_+:=a_0+\varepsilon,
    \qquad
    a_-:=a_0-\varepsilon.
\]
Integrating \eqref{eq:D-slope} gives
\begin{equation}
    D(a_+)\ge\gamma\varepsilon,
    \qquad
    D(a_-)\le-\gamma\varepsilon.
    \label{eq:gaps}
\end{equation}
Thus in the $a_+$ instance the region $c\ge c_0$ is optimal, and every choice $c<c_0$ incurs expected excess loss at least $\gamma\varepsilon$. In the $a_-$ instance the region $c<c_0$ is optimal, and every choice $c\ge c_0$ incurs expected excess loss at least $\gamma\varepsilon$. These statements remain true for arbitrary predictors because each payment-region comparison uses the prediction loss minimized over $|\nu|\le S$.

We again strengthen the learner by revealing $X_t$ after each round. Let $P_+^T$ and $P_-^T$ denote the enhanced transcript laws. Conditional on $X_t$ and the chosen payment, the label channel and the measurement-noise channel are identical in the two instances. The clean feature is fresh and independent of the past, so the chain rule gives
\begin{equation}
    \KL(P_+^T\|P_-^T)
    =
    T\,
    \KL\!\left(
        \mathcal N(0,a_+)
        \,\middle\|\,
        \mathcal N(0,a_-)
    \right).
    \label{eq:known-transcript-kl}
\end{equation}
For $\varepsilon\le a_0/2$, both variances are at least $a_0/2$. By Lemma~\ref{lem:gaussian-kl} with $u=a_+$, $v=a_-$, and $\underline{s}=a_0/2$,
\begin{equation}
    \KL\!\left(
        \mathcal N(0,a_+)
        \,\middle\|\,
        \mathcal N(0,a_-)
    \right)
    \le
    \frac{4\varepsilon^2}{a_0^2}.
    \label{eq:known-one-kl}
\end{equation}

Choose
\[
    \varepsilon=\frac{\eta}{\sqrt T},
\]
where $\eta>0$ is a fixed constant small enough that $4\eta^2/a_0^2\le1/8$. For sufficiently large $T$, we also have $\varepsilon\le\min\{\rho,a_0/2\}$. Then
\[
    \KL(P_+^T\|P_-^T)\le\frac18,
    \qquad
    \TV(P_+^T,P_-^T)\le\frac14.
\]

Let
\[
    M_T:=\sum_{t=1}^T\ind\{c_t\ge c_0\}.
\]
In the $a_+$ instance the number of wrong-region decisions is $T-M_T$, whereas in the $a_-$ instance it is $M_T$. Since $0\le M_T\le T$,
\begin{align}
    \E_+[T-M_T]+\E_-[M_T]
    &=T-\bigl(\E_+[M_T]-\E_-[M_T]\bigr)\\
    &\ge T\bigl(1-\TV(P_+^T,P_-^T)\bigr)\\
    &\ge\frac{3T}{4}.
\end{align}
Therefore at least one of the two instances has expected wrong-region count at least $3T/8$. By \eqref{eq:gaps},
\begin{align}
    \max\{\Reg_T^\pi(a_+), \Reg_T^\pi(a_-)\}
    &\ge
    \frac{3\gamma}{8}T\varepsilon\\
    &=
    \frac{3\gamma\eta}{8}\sqrt T.
\end{align}
Since $\pi$ was arbitrary and $a_\pm\in[a_0-\rho,a_0+\rho]$, this proves \eqref{eq:known-lb}. The same fixed model bounds prove Theorem~\ref{thm:known-lower-app}.
\end{proof}

\begin{remark}[Source of the known-profile lower bound]
The known-profile construction leaves only the clean-feature variance $a\in\{a_-,a_+\}$ unknown. The measurement profile, payment cost, structural logistic parameter, and predictor radius may all be revealed, and the clean feature may be observed after each round. Thus the $\sqrt T$ lower bound arises from the difficulty of distinguishing two nearby instances whose optimal payment regions are different, rather than from estimating the logistic parameter or recovering the clean feature.
\end{remark}

\subsection{Proof of Theorem~\ref{thm:unknown-lower}: unknown profile}\label{app:unknown-lower}

\paragraph{Proof sketch.}
We first use the inverse regularity of the scalar logistic loss to construct a baseline Gaussian variance profile whose optimized population loss is exactly flat over the entire payment interval. We then create $K$ alternative profiles, each obtained by inserting a small improvement of depth $\delta_K\asymp 1/K$ inside one of $K$ disjoint payment regions. Under alternative $k$, only observations collected inside the hidden region $A_k$ are informative for distinguishing it from the flat baseline. A pigeonhole argument identifies a region that is visited only $O(T/K)$ times under the baseline, and the resulting transcript KL divergence is of order $T/K^3$. Choosing $K\asymp T^{1/3}$ keeps this divergence bounded, so the learner cannot reliably localize the improved region and must spend a constant fraction of the horizon outside it. Since each such round incurs a gap of order $1/K$, the resulting regret is $\Omega(T/K)=\Omega(T^{2/3})$. The constructed Gaussian profiles are coupled through Brownian motion so that they satisfy the required ordered additive degradation assumption.

\begin{theorem}[Unknown decreasing Lipschitz Gaussian profile]
\label{thm:unknown-lower-app}
Fix any $S>0$. There exist fixed constants
\[
    a>0,\quad
    \theta\in (0,S], \quad
    \lambda>0,\quad
    0<s_{\min}<s_{\max}<\infty,
\]
and $L, c>0$, independent of $T$, such that the following holds. Let
\[
    \mathcal S_L
    :=
    \left\{
        s:[0,1]\to[s_{\min},s_{\max}]:
        s\text{ is nonincreasing and }L\text{-Lipschitz}
    \right\}.
\]
For every sufficiently large $T$ and every possibly randomized learning policy $\pi$,
\begin{equation}
    \sup_{s\in\mathcal S_L}
    \Reg_T^\pi (s)
    \ge
    cT^{2/3}.
    \label{eq:unknown-lb}
\end{equation}
The conclusion remains true if $a,\theta,S,\lambda,s_{\min},s_{\max}$ are revealed to the learner and the clean feature $X_t$ is additionally revealed after every round. Every Gaussian noise profile in this class satisfies Assumption~\ref{ass:degradation}. In particular, one may take $s_{\min}=1$ and $s_{\max}=2$.
\end{theorem}

\begin{proof}
Fix any $a>0$, $0 < \theta \leq S$, and $0<s_{\min}<s_{\max}<\infty$. Write $H=H_a$, where $H_a(s)=H(a,s)$, and define
\[
    C:=H(s_{\max}),
    \qquad
    \lambda:=H(s_{\max})-H(s_{\min})>0.
\]
All profiles constructed below satisfy Assumption~\ref{ass:sg} with any
fixed $B>0$ and
\[
    R\ge\max\{1,\sqrt a,\sqrt{s_{\max}}\}.
\]
The payment weight $\lambda$ and these model bounds are fixed independently of $T$.
Define the baseline variance profile
\begin{equation}
    s_0(c)
    :=
    H^{-1}(C-\lambda c),
    \qquad c\in[0,1].
    \label{eq:s0}
\end{equation}
By construction,
\begin{equation}
    H(s_0(c))+\lambda c=C
    \qquad\text{for every }c\in[0,1].
    \label{eq:flat}
\end{equation}
Thus the baseline optimized population loss is flat over the entire payment interval.

Let $K\ge1$ be an integer. Partition $J=[1/4,3/4]$ into $K$ cells of length $1/(2K)$, and let $z_k$ be the midpoint of cell $k$. Define
\begin{equation}
    \phi_k(c)
    :=
    \bigl(1-4K|c-z_k|\bigr)_+,
    \qquad k\in[K],
    \label{eq:tent}
\end{equation}
and let
\[
    A_k:=\{c\in[0,1]:\phi_k(c)>0\}.
\]
The sets $A_1,\ldots,A_K$ are pairwise disjoint, $0\le\phi_k\le1$, $\phi_k(z_k)=1$, and $\phi_k$ is $4K$-Lipschitz.

Set
\[
    \alpha:=\frac{\lambda}{16},
    \qquad
    \delta_K:=\frac{\alpha}{K},
\]
and define
\begin{equation}
    q_k(c)
    :=
    C-\lambda c-\delta_K\phi_k(c),
    \qquad
    s_k(c)
    :=
    H^{-1}(q_k(c)).
    \label{eq:sk}
\end{equation}
We verify that $s_k$ is well defined and belongs to a common decreasing Lipschitz class. If $c\notin A_k$, then $q_k(c)=C-\lambda c$ lies in the domain of $H^{-1}$. If $c\in A_k\subset[1/4,3/4]$, then
\[
    C-\lambda c
    \ge
    H(s_{\min})+\frac{\lambda}{4},
\]
while $\delta_K\le\lambda/16$. Hence
\[
    q_k(c)
    \ge
    H(s_{\min})+\frac{3\lambda}{16}
    >H(s_{\min}),
\]
and clearly $q_k(c)\le H(s_{\max})$. Therefore $s_k(c)$ is well defined and remains in $[s_{\min},s_{\max}]$.

For $c_1<c_2$, using the $4K$-Lipschitz property of $\phi_k$,
\begin{align}
    q_k(c_2)-q_k(c_1)
    &=-\lambda(c_2-c_1)
      -\delta_K\bigl(\phi_k(c_2)-\phi_k(c_1)\bigr)\\
    &\le
      \left(-\lambda+4K\delta_K\right)(c_2-c_1)\\
    &=-\frac{3\lambda}{4}(c_2-c_1)<0.
\end{align}
Thus $q_k$ is decreasing, and $H^{-1}$ is increasing, $s_k (\cdot)= H^{-1}(q_k(\cdot)) $ is decreasing. Moreover,
\[
    |q_k(c_2)-q_k(c_1)|
    \le
    \frac{5\lambda}{4}|c_2-c_1|,
\]
so Corollary~\ref{cor:inverse} gives
\begin{equation}
    \operatorname{Lip}(s_k)
    \le
    \frac{5\lambda}{4g}
    =:L,
    \label{eq:L}
\end{equation}
where $L$ is independent of $k,K$, and $T$. The baseline profile $s_0$ also belongs to this class.

We next verify Assumption~\ref{ass:degradation}. For each round $t$, let $B_t$ be a standard Brownian motion with these motions independent across rounds and independent of all clean features and labels, and define
\[
    N_t(c):=B_t(s_k(c)).
\]
Then $N_t(c)\sim\mathcal N(0,s_k(c))$. If $c_1\le c_2$, then $s_k(c_1)\ge s_k(c_2)$ and
\[
    N_t(c_1)
    =
    N_t(c_2)
    +
    \bigl(B_t(s_k(c_1))-B_t(s_k(c_2))\bigr),
\]
where the increment is independent of $N_t(c_2)$ and has variance $s_k(c_1)-s_k(c_2)$.

Write $L_k(c,\nu):=F(a,s_k(c),\nu)+\lambda c$ and
$L_k^\star(c):=\min_{|\nu|\le S}L_k(c,\nu)$. By \eqref{eq:sk}, the optimized population loss in alternative $k$ is exactly
\begin{equation}
    L_k^\star(c)
    =
    H(s_k(c))+\lambda c
    =
    C-\delta_K\phi_k(c).
    \label{eq:loss-dent}
\end{equation}
Hence
\[
    L_k^\star:=\min_{c\in[0,1]}L_k^\star(c)=C-\delta_K,
\]
attained at $z_k$, and for every $c\notin A_k$,
\begin{equation}
    L_k^\star(c)-L_k^\star
    =
    \delta_K.
    \label{eq:outside-gap}
\end{equation}

Fix any possibly randomized policy $\pi$. It remains to show that the learner cannot identify the hidden region $A_k$ quickly enough. We work under the stronger feedback model in which $X_t$ is revealed after each round. Let $P_0^T$ and $P_k^T$ denote the laws of the complete enhanced transcript under the baseline and alternative $k$, respectively. Conditional on the past and on the chosen payment $c_t$, the clean-feature law and the conditional label law are identical under $P_0^T$ and $P_k^T$. Once $X_t$ is revealed, the noise realization $N_t(c_t)=\widehat X_t-X_t$ is observed. Therefore the one-step conditional KL divergence equals
\[
    \KL\!\left(
        \mathcal N(0,s_0(c_t))
        \,\middle\|\,
        \mathcal N(0,s_k(c_t))
    \right).
\]

By Lemma~\ref{lem:gaussian-kl}, for any $u,v\ge s_{\min}$,
\[
    \KL\!\left(\mathcal N(0,u)\,\middle\|\,\mathcal N(0,v)\right)
    \le
    \frac{(u-v)^2}{4s_{\min}^2}.
\]
Together with \eqref{eq:inverse-lip},
\[
    |s_0(c)-s_k(c)|
    \le
    \frac{\delta_K}{g}\phi_k(c).
\]
Thus
\begin{equation}
    \KL\!\left(
        \mathcal N(0,s_0(c))
        \,\middle\|\,
        \mathcal N(0,s_k(c))
    \right)
    \le
    C_1\delta_K^2\ind\{c\in A_k\},
    \qquad
    C_1:=\frac{1}{4s_{\min}^2g^2}.
    \label{eq:local-kl}
\end{equation}

Define
\[
    N_k(T)
    :=
    \sum_{t=1}^T\ind\{c_t\in A_k\}.
\]
The chain rule for relative entropy applied to the adaptive transcript gives
\begin{equation}
    \KL(P_0^T\|P_k^T)
    \le
    C_1\delta_K^2\E_0[N_k(T)].
    \label{eq:adaptive-kl}
\end{equation}
No additional term arises from adaptive action selection because, conditional on a realized history, the same randomized decision rule is used in both environments.

Since the $A_k$ are pairwise disjoint,
\[
    \sum_{k=1}^K N_k(T)\le T.
\]
Hence there exists $k^\star\in[K]$ such that
\begin{equation}
    \E_0[N_{k^\star}(T)]\le\frac{T}{K}.
    \label{eq:least-visited}
\end{equation}
For this index,
\[
    \KL(P_0^T\|P_{k^\star}^T)
    \le
    C_1\alpha^2\frac{T}{K^3}.
\]
Choose
\[
    K=\left\lceil\kappa T^{1/3}\right\rceil
\]
with a fixed $\kappa$ sufficiently large that $C_1\alpha^2/\kappa^3\le1/8$. Then
\[
    \KL(P_0^T\|P_{k^\star}^T)\le\frac18,
\]
and Pinsker's inequality yields
\[
    \TV(P_0^T,P_{k^\star}^T)\le\frac14.
\]
Because $0\le N_{k^\star}(T)\le T$,
\[
    \E_{k^\star}[N_{k^\star}(T)]
    \le
    \E_0[N_{k^\star}(T)]
    +T\TV(P_0^T,P_{k^\star}^T)
    \le
    \frac{T}{K}+\frac{T}{4}.
\]
For sufficiently large $T$, $K\ge4$, so
\[
    \E_{k^\star}[N_{k^\star}(T)]\le\frac{T}{2}.
\]

For every selected predictor $\nu_t$,
\[
    L_{k^\star}(c_t,\nu_t)
    \ge
    L_{k^\star}^\star(c_t).
\]
Therefore \eqref{eq:outside-gap} implies
\begin{align}
    \Reg_T^\pi (s_{k^\star})
    &\ge
    \delta_K\,
    \E_{k^\star}[T-N_{k^\star}(T)]\\
    &\ge
    \frac{\delta_KT}{2}
    =
    \frac{\alpha T}{2K}
    =
    \Omega(T^{2/3}).
\end{align}
Since $\pi$ is arbitrary, this proves \eqref{eq:unknown-lb}. Taking $s_{\min}=1$ and $s_{\max}=2$ proves Theorem~\ref{thm:unknown-lower-app}.
\end{proof}

\begin{remark}[Source of the lower bound]
Theorem~\ref{thm:unknown-lower-app} continues to hold when all fixed model parameters are known and the clean feature is revealed after every round. Thus the $T^{2/3}$ obstruction is caused by localization of an unknown payment profile, rather than by structural logistic-parameter estimation or by post hoc separation of the clean feature from measurement error.
\end{remark}

\end{document}